\documentclass[nohyperref]{article}

\usepackage{microtype}
\usepackage{graphicx}
\usepackage{subfigure}
\usepackage{booktabs} 
\usepackage{multirow}
\usepackage{caption}
\usepackage{color}
\usepackage{wrapfig}
\usepackage{float}
\usepackage{url}
\newenvironment{proofsketch}{%
  \proof}{\endproof}

\definecolor{amethyst}{rgb}{0.6, 0.4, 0.8}

\newcommand\revision[1]{#1}

\usepackage{hyperref}

\usepackage{amsmath,amsfonts,bm}

\def\eqref#1{equation~\ref{#1}}

\def\1{\bm{1}}

\def\rb{{\textnormal{b}}}

\def\vb{{\bm{b}}}

\DeclareMathAlphabet{\mathsfit}{\encodingdefault}{\sfdefault}{m}{sl}
\SetMathAlphabet{\mathsfit}{bold}{\encodingdefault}{\sfdefault}{bx}{n}

\newcommand{\R}{\mathbb{R}}

\newcommand{\Dt}{\Delta\theta}

\let\ab\allowbreak

\DeclareMathOperator{\dom}{dom}
\newcommand{\Lcal}{\mathcal{L}}
\newcommand{\RR}{\mathbb{R}} 
\newcommand{\T}{^\top}

\usepackage[accepted]{icml2022}

\usepackage{amsmath,amssymb,amsthm,amscd,amstext}
\usepackage{mathtools}

\usepackage[capitalize,noabbrev]{cleveref}

\theoremstyle{plain}
\newtheorem{theorem}{Theorem}[section]
\newtheorem{proposition}[theorem]{Proposition}
\newtheorem{lemma}[theorem]{Lemma}
\newtheorem{claim}[theorem]{Claim}

\theoremstyle{definition}

\newtheorem{assumption}[theorem]{Assumption}
\newtheorem{axiom}[theorem]{Axiom}
\theoremstyle{remark}

\newcommand{\ourmethod}{Nash-MTL}
\newtheorem*{theorem*}{Theorem}

\makeatletter
\newcommand*\bigcdot{\mathpalette\bigcdot@{1.}}
\newcommand*\bigcdot@[2]{\mathbin{\vcenter{\hbox{\scalebox{#2}{$\m@th#1\bullet$}}}}}
\makeatother

\usepackage[textsize=tiny]{todonotes}

\icmltitlerunning{Multi-Task Learning as a Bargaining Game}

\begin{document}

\twocolumn[
\icmltitle{Multi-Task Learning as a Bargaining Game}



\icmlsetsymbol{equal}{*}

\begin{icmlauthorlist}
\icmlauthor{Aviv Navon}{equal,biu}
\icmlauthor{Aviv Shamsian}{equal,biu}
\icmlauthor{Idan Achituve}{biu}
\icmlauthor{Haggai Maron}{nvidia}
\icmlauthor{Kenji Kawaguchi}{nus} \\
\icmlauthor{Gal Chechik}{biu,nvidia}
\icmlauthor{Ethan Fetaya}{biu}
\end{icmlauthorlist}

\icmlaffiliation{biu}{Bar-Ilan University, Ramat Gan, Israel}
\icmlaffiliation{nvidia}{Nvidia, Tel-Aviv, Israel}
\icmlaffiliation{nus}{National University of Singapore}

\icmlcorrespondingauthor{Aviv Navon}{aviv.navon@biu.ac.il}
\icmlcorrespondingauthor{Aviv Shamsian}{aviv.shamsian@live.biu.ac.il}

\icmlkeywords{Machine Learning, ICML}

\vskip 0.3in
]



\printAffiliationsAndNotice{\icmlEqualContribution} 

\begin{abstract}

In Multi-task learning (MTL), a joint model is trained to simultaneously make predictions for several tasks. Joint training reduces computation costs and improves data efficiency; however, since the gradients of these different tasks may conflict, training a joint model for MTL often yields lower performance than its corresponding single-task counterparts. 
A common method for alleviating this issue is to combine per-task gradients into a joint update direction using a particular heuristic. In this paper, we propose viewing the gradients combination step as a bargaining game, where tasks negotiate to reach an agreement on a joint direction of parameter update. Under certain assumptions, the bargaining problem has a unique solution, known as the \emph{Nash Bargaining Solution}, which we propose to use as a principled approach to multi-task learning. We describe a new MTL optimization procedure, Nash-MTL, and derive theoretical guarantees for its convergence. Empirically, we show that Nash-MTL achieves state-of-the-art results on multiple MTL benchmarks in various domains.
\end{abstract}
\section{Introduction}
\label{sec:Inro}


\begin{figure*}[t]
    \centering
    
    \includegraphics[width=1.0\linewidth]{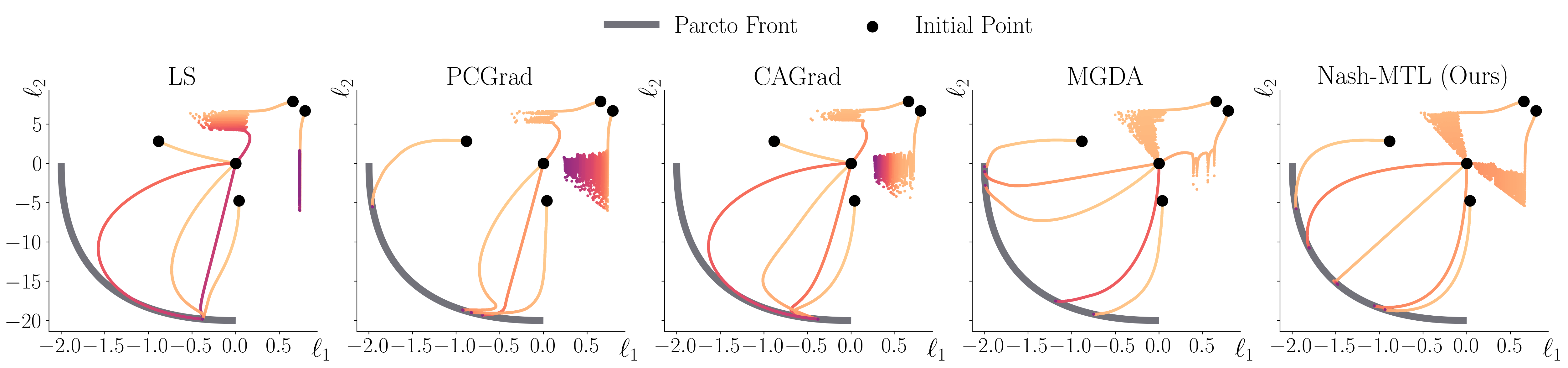}
    
    \caption{\textit{Illustrative example}: Optimization trajectories in loss space. Shown are 5 different initializations (black dots $\bigcdot$), and their trajectories are colored from orange to purple. Losses have a large difference in scale. See Appendix~\ref{sec:appendix_illustrative} for details. For linear scalarization (LS), PCGrad, and CAGrad, the optimization process is controlled by the gradient of $\ell_2$, since it has a larger magnitude, resulting in imbalanced solutions between tasks (mostly ending at the bottom right). These three methods also fail to converge to an optimal solution for the rightmost initialization points. In contrast, MGDA is inclined towards the task with the smallest gradient magnitude ($\ell_1$). Our method, \textit{\ourmethod{}}, is invariant to changes in loss scale and produces solutions that are well balanced across the Pareto front.
    }
    \label{fig:toy}
\end{figure*}

In many real-world applications, one needs to solve several tasks simultaneously using limited computational or data resources. For example, perception for autonomous vehicles requires lane detection, object detection, and free-space estimation, which must all run in parallel and in real-time.  This is normally solved via multi-task learning (MTL), where one model is jointly trained on several learning tasks~\cite{caruana1997multitask,ruder2017overview,crawshaw2020multi}. Multi-task learning was also shown to improve generalization in theory \cite{Baxter00} and in practice \citep[e.g., auxiliary learning,][]{liu2019self,achituve2021self,NavonAMCF21}.

Unfortunately, MTL often causes performance degradation compared to single-task models~\cite{tasks_together}. 
A main reason for such degradation is gradients conflict ~\cite{yu2020gradient,wang2020gradient,liu2021conflict}. These per-task gradients may have conflicting directions or a large difference in magnitudes, with the largest gradient dominating the update direction.
The degraded performance of MTL due to poor training, compared with its potential to improve performance due to better generalization, has a major impact on many real-world systems.
Improving MTL optimization algorithms is therefore an important task with significant implications to many systems. 

Currently, most MTL optimization algorithms ~\cite{sener2018multi,yu2020gradient,liu2021conflict} follow a general scheme. First, compute the gradients for all tasks $g_1,...,g_K$. Next, combine those gradients into a joint direction, $\Delta=\mathcal{A}(g_1,...,g_K)$ using an aggregation algorithm $\mathcal{A}$. Finally, update model parameters using a single-task optimization algorithm, replacing the gradients with $\Delta$. Multiple heuristics were proposed for the aggregation algorithm $\mathcal{A}$. However, to the best of our knowledge, a principled, axiomatic, approach to gradient aggregation is still missing.


Here we address the gradient combination step by viewing  it as a cooperative bargaining game \cite{coop_barg}. Each task represents a player, whose utility is derived from its gradient, and players negotiate to reach an agreed direction. This formulation allows us to use results from game theory literature that analyze this problem from an axiomatic perspective. In his seminal paper, \citet{nash} presented an axiomatic approach to the bargaining problem and showed that under certain axioms, the bargaining problem has a unique solution known as the \emph{Nash Bargaining Solution}. This solution is known to be proportionally fair
, where any alternative will have a negative average relative change. This proportionally fair update allows us to find a solution that works for all tasks without being dominated by a single large gradient.

Building on Nash's results, we propose a novel MTL optimization algorithm, named \textit{\ourmethod{}}, where the gradients are combined at each step using the Nash bargaining solution. We first characterize the Nash bargaining solution for MTL and derive an efficient algorithm to approximate its value. Then, we analyze our approach theoretically and establish convergence guarantees in the convex and nonconvex cases. Finally, we show empirically that our Nash-MTL approach achieves state-of-the-art results on four MTL benchmarks on a variety of challenges ranging from computer vision and quantum chemistry to reinforcement learning. \revision{To support future research and the reproducibility of the results, we make our source code publicly available at: \textcolor{magenta}{\url{https://github.com/AvivNavon/nash-mtl}}.}

\section{Background}
\label{sec:Background}
\subsection{Pareto Optimality}
Optimization for MTL is a specific case of multiple-objective optimization (MOO). Given objective functions $\ell_1,...,\ell_K$, the performance of solution a $x$ is measured by the vector of objective values $(\ell_1(x),...,\ell_K(x))$. One main property of MOO is that since there is no natural linear ordering on vectors it is not always possible to compare solutions so there is no clear optimal value. 

We say that a solution $x$ dominates $x'$ if it is better on one or more objectives and not worse on any other objectives. A solution that is not dominated by any other is called \emph{Pareto optimal}, and the set of all such solutions is called the \emph{Pareto front}. It is important to note that there is no clear way to select between different Pareto optimal solutions without additional assumptions or prior about the user preferences~\cite{navon2021learning}. For non-convex problems, a point is defined as local Pareto optimal if it is Pareto optimal in some open set containing it. \revision{Further, a point is called \emph{Pareto stationary} if there exists a convex combination of the gradients at this point that equals zero. Pareto stationarity is a necessary condition for Pareto optimality.}


\subsection{Nash Bargaining Solution}
We provide a brief background on cooperative bargaining games and the Nash bargaining solution, see \citet{coop_barg} for more details. In a bargaining problem, we have $K$ players, each with their own utility function $u_i:A\cup\{D\}\rightarrow\R$, which they wish to maximize. $A$ is the set of possible agreements and $D$ is the disagreement point which the players default to if they fail to reach an agreement. We define the set of possible payoffs as $U=\{(u_1(x),...,u_K(x)):\,x\in A\}\subset\R^K$ and $d=(u_1(D),...,u_K(D))$. We assume $U$ is convex, compact and that there exists a point in $U$ that strictly dominates $d$, namely there exists a $u\in U$ such that $\forall i:u_i>d_i$.

\citet{nash} showed that for such payoff set $U$, the two-player bargaining problem has a unique solution
that satisfies the following properties or axioms
: Pareto optimality, symmetry, independence of irrelevant alternatives, and invariant to affine transformations. This was later extended to multiple players \cite{game_theory}.
\begin{axiom}
\textbf{Pareto optimality:} The agreed solution must not be dominated by another option, i.e. there cannot be any other agreement that is better for at least one player and not worse for any of the players.
\end{axiom}
As it is a cooperative game, it makes little sense that the players will curtail another player without any personal gains, so it is natural to assume the agreed solution will not be dominated by another. 
\begin{axiom}
\textbf{Symmetry:} The solution should be invariant to permuting the order of the players.
\end{axiom}
\begin{axiom}
\textbf{Independence of irrelevant alternatives (IIA):} If we enlarge the of possible payoffs to 
$\tilde{U}\supsetneq U$, and the solution is in the original set $U$, $u^*\in U$, then the agreed point when the set of possible payoffs is $U$ will stay $u^*$.
\end{axiom}
\begin{axiom}
\textbf{Invariance to affine transformation:} If we transform each utility function $u_i(x)$ to $\tilde{u}_i(x)=c_i\cdot u_i(x)+b_i$ with $c_i>0$ then if the original agreement had utilities $(y_1,...,y_k)$ the agreement after the transformation has utilities $(c_1y_1+b_1,...,c_ky_k+b_k)$
\end{axiom}
We argue that in the MTL setting, it is natural to require axioms 2.1-2.3. Axiom 2.4, in our mind, is the only non-natural assumption used by the Nash bargaining solution in the context of MTL. We argue that indeed it is a desired property that is helpful for MTL. Axiom 2.4 means that the solution does not take into account the gradients' norms but rather treats all of them the same, as if they were normalized. 
Without enforcing this assumption, the solution can easily be dominated by a single direction (see Figure~\ref{fig:toy}). We further validate the importance of this assumption by investigating a scale-invariant baseline in Section~\ref{sec:Exp}.

The unique point satisfying all these axioms is called the Nash bargaining solution and is given as 
\begin{align}\label{eq:nash_barg}
    u^*=&\arg\max_{u\in U}\sum_i\log(u_i-d_i)\\
    &s.t.\,\,\forall i:\,u_i>d_i\nonumber
\end{align}

\section{Method}
\label{sec:Method}

We now describe our Nash-MTL method in detail. We first formalize the gradient combination step as a bargaining game and analyze the Nash bargaining solution for this game. We then describe our algorithm to approximate the solution efficiently. We note that the computational cost of that approximation 
is critical because this approximation is executed for each gradient update. \revision{To simplify the notation, we do not distinguish between shared and task-specific parameters. We note, however, that task-specific parameters have no contribution to the Nash bargaining solution calculation.}

\subsection{Nash Bargaining Multi-Task Learning}
Given an MTL optimization problem and parameters $\theta$, we search for an update vector $\Delta\theta$ in the ball of radius $\epsilon$ centered around zero, $B_\epsilon$. We frame this as a bargaining problem with the agreement set $B_\epsilon$ and the disagreement point at $0$, i.e., staying at the current parameters $\theta$. We define the utility function for each player as $u_i(\Dt)=g_i\T \Dt$ where $g_i$ is the gradient of the loss of task $i$ at $\theta$. We note that since the agreement set is compact and convex and the utilities are linear then the set of possible payoffs is also compact and convex.

Our main assumption, besides the ones used by Nash, is that if $\theta$ is not Pareto stationary then the gradients are linearly independent (see further discussion on this assumption in Section \ref{sec:Analysis}). Under this assumption, we also have that the disagreement point, $\Dt=0$ is dominated by another in $B_\epsilon$. We now show that if $\theta$ is not on the Pareto front, the unique Nash bargaining solution has the following form: 

\begin{claim}
Let $G$ be the $d\times K$ matrix whose columns are the gradients $g_i$. The solution to $\arg\max_{\Dt\in B_\epsilon}\sum_i \log(\Dt\T g_i)$ is (up to scaling) $\sum_i\alpha_i g_i$ where $\alpha\in\mathbb{R}^K_+$ is the solution to $G\T G\alpha=1/\alpha$ where $1/\alpha$ is the element-wise reciprocal.  

\end{claim}
\begin{proof}
The derivative of this objective is $\sum_{i=1}^K\frac{1}{\Dt\T g_i}g_i$. For all vectors $\Dt$ such that $\forall i:\Dt^Tg_i>0$ the utilities are monotonically increasing with the norm of $\Dt$. Thus, from the Pareto optimality assumption by Nash, the optimal solution has to be on the boundary of $B_\epsilon$. From this we see that the gradient at the optimal point $\sum_{i=1}^K\frac{1}{\Dt\T g_i}g_i$ must be in the radial direction, i.e., $\sum_{i=1}^K\frac{1}{\Dt\T g_i}g_i\parallel\Dt$ or $\sum_{i=1}^K\frac{1}{\Dt\T g_i}g_i=\lambda\Dt$. Since the gradients are independent we must have $\Dt=\sum_i \alpha_ig_i$ and $\forall i:\frac{1}{\Dt\T g_i}=\lambda \alpha_i$ or $\forall i:{\Dt\T g_i}=\frac{1}{\lambda \alpha_i}$. As the inner product must be positive for a descent direction we can conclude $\lambda>0$; we set $\lambda=1$ to ascertain the direction of $\Dt$ (the norm might be larger then $\epsilon$). Now finding the bargaining solution is reduced to finding $\alpha\in\mathbb{R}^K$ with $\alpha_i>0$ such that $\forall i:{\Dt\T g_i}=\sum_j \alpha_jg_j\T g_i=\frac{1}{\alpha_i}$. This is equivalent to requiring that $G\T G\alpha=1/\alpha$ where $1/\alpha$ is the element-wise reciprocal.
\end{proof}

We now provide some intuition for this solution. First, if all $g_i$ are orthogonal we get $\alpha_i=1/||g_i||$ and $\Dt=\sum\frac{g_i}{||g_i||}$ which is the obvious scale invariant solution. When they are not orthogonal, we get 
\begin{equation}\label{eq:intuition}
    \alpha_i||g_i||^2+\sum_{j\neq i}\alpha_jg_j\T g_i=1/\alpha_i
\end{equation}
We can consider $\sum_{j\neq i}\alpha_jg_j\T g_i=\left(\sum_{j\neq i}\alpha_jg_j\right)\T g_i$ as the interaction between task $i$ and the other tasks; If it is positive there is a positive interaction and the other gradients aid the $i$'th task, and if it is negative they hamper it. When there is a negative interaction, the LHS of Eq. \ref{eq:intuition} decreases and as a result, $\alpha_i$ increases to compensate for it. Conversely, where there is a positive interaction $\alpha_i$ will decrease.


\begin{algorithm}[t]
    \caption{\ourmethod{}}\small\label{alg:nash-mtl}
    \begin{algorithmic}[H]
    \STATE {\bfseries Input:} $\theta^{(0)}$ -- initial parameter vector, $\{\ell_i\}_{i=1}^K$ -- differentiable loss functions, $\eta$ -- learning rate
    \FOR{$t=1,...,T$}
    \STATE Compute task gradients $g^{(t)}_i=\nabla_{\theta^{(t-1)}}\ell_i$
    \STATE Set $G^{(t)}$ the matrix with columns $g^{(t)}_i$
    \STATE Solve for $\alpha$: $(G^{(t)})\T G^{(t)}\alpha=1/\alpha$ to obtain $\alpha^{(t)}$
    \STATE Update the parameters $\theta^{(t)}=\theta^{(t)}-\eta G^{(t)}\alpha^{(t)}$
    \ENDFOR
    \STATE {\bfseries Return:} $\theta^{(T)}$
    \end{algorithmic}
\end{algorithm}

\begin{figure*}[t]
    \centering
    
    \includegraphics[width=.95\linewidth]{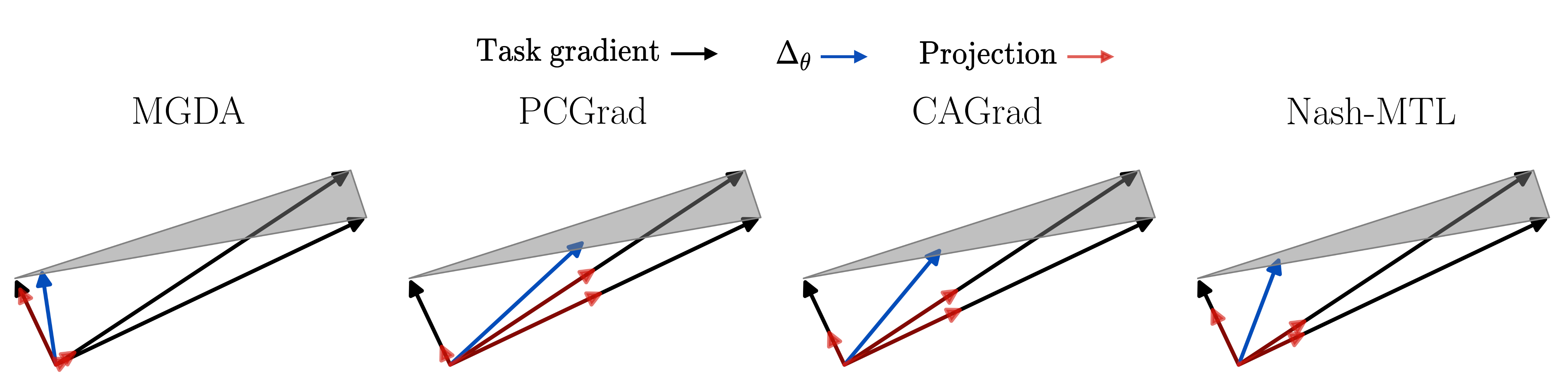}
    
    \caption{\textit{Visualization of the update direction}: We show the update direction {(blue)} obtained by various methods on three gradients in $\R^3$. We rescaled the returned vectors for better visibility, showing only the direction. We further show the size of the projection {(red)} of the update to each gradient direction {(black)}. \ourmethod{} produce an update direction with the most balanced projections.}
    \label{fig:update}
\end{figure*}

\subsection{Solving $\mathbf{G\T G\alpha }=1/\alpha$}
\label{sec:ccp}

Here we describe how to efficiently approximate the optimal solution for $G\T G\alpha=1/\alpha$ through a sequence of convex optimization problems. We define a $\beta_i(\alpha)=g_i\T G\alpha$, and wish to find $\alpha$ such that $\alpha_i=1/\beta_i$ for all $i$, or equivalently $\log(\alpha_i)+\log(\beta_i(\alpha))=0$. Denote $\varphi_i(\alpha)=\log(\alpha_i)+\log(\beta_i)$ and $\varphi(\alpha)=\sum_i \varphi_i(\alpha)$. 
With that, our goal is to find a non-negative $\alpha$ such that $\varphi_i(\alpha)=0$ for all $i$. We can write this as the following optimization problem
\begin{align}
    \min_{\alpha} &\sum_i \varphi_i(\alpha) \\
    \text{s.t.} \forall i, \quad & -\varphi_i(\alpha) \leq 0 \nonumber\\
    &\alpha_i > 0 \nonumber \quad.
\end{align} 
The constraints in this problem are convex and linear and the objective is concave. We first try to solve the following convex surrogate objective
\begin{align}\label{eq:ccp_opt_take_1}
    \min_{\alpha} &\sum_i \beta_i(\alpha) \\
    \text{s.t.} \forall i, \quad & -\varphi_i(\alpha) \leq 0 \nonumber\\
    &\alpha_i > 0 \quad. \nonumber 
\end{align}
Here, we minimize $\sum_i\beta_i$ under the constraint $\beta_i=g_i\T G\alpha \geq 1/\alpha_i$. While this objective is not equivalent to the original problem, we found it very useful. In many cases, it produces exact solutions with $\varphi(\alpha)=0$ as required. 

To further improve our approximation, we considered the following problem,
\begin{align}\label{eq:ccp_opt}
    \min_{\alpha} &\sum_i \beta_i(\alpha) + \varphi(\alpha) \\
    \text{s.t.} \forall i, \quad & -\varphi_i(\alpha) \leq 0 \nonumber\\
    &\alpha_i > 0 \quad.\nonumber
\end{align}
Adding $\varphi(\alpha)$ to the objective may further reduce it, moving it closer to zero; however, it renders the problem to be non-convex. 
Despite that, our solution can be improved iteratively by replacing the concave term $\varphi(\alpha)$ with its first-order approximation $\tilde{\varphi}_\tau(\alpha)=\varphi(\alpha^{(\tau)}) + \nabla\varphi(\alpha^{(\tau)})\T (\alpha-\alpha^{(\tau)})$. Where, $\alpha^{(\tau)}$ is the solution at iteration $\tau$. Note that we replace $\varphi$ with $\tilde{\varphi}$ only in the objective and keep $\varphi(\alpha)$ as is in the constraint: i.e., $\min_{\alpha} \sum_i \beta_i(\alpha) + \tilde \varphi_\tau(\alpha) \text{ s.t. } -\varphi_i(\alpha) \leq 0 \text { and } \alpha_i > 0$ for all $i$. This sequential optimization approach is a variation of the concave-convex procedure (CCP) \citep{yuille2003concave,lipp2016variations}. Therefore the sequence $\{\alpha^{(\tau)}\}_\tau$ converges to a critical point of the original non-convex problem in Eq.~\ref{eq:ccp_opt} based on previous theory of CCP by~\citet{sriperumbudur2009convergence}. Moreover, since we do not modify the constraint, $\alpha^{(\tau)}$ always satisfies the constraint of the original problem for any $\tau$. Finally, the following proposition shows that original objective monotonically decreases with $\tau$:
\begin{proposition}
Denote the objective for the optimization problem in Eq.~\ref{eq:ccp_opt} by $\phi(\alpha)=\sum_i\beta_i(\alpha)+\varphi(\alpha)$. Then, $\phi\left(\alpha^{(\tau+1)}\right)\leq \phi\left(\alpha^{(\tau)}\right)$ for all $\tau\geq 1$.
\end{proposition}

We provide proof and further discussion in Appendix~\ref{sec:Appenix_proofs}. In practice, we limit the sequence of CCP to 20 in all experiments, with the exception of Section~\ref{sec:mt10} for which we use a single step. We found the improved solution to have a limited effect on the MTL performance (see Appendix~\ref{sec:Appendix_ccp_steps}).

\subsection{Practical Speedup}
\label{sec:speedup}
One shortcoming of many leading MTL methods is that all task gradients are required for obtaining the joint update direction. When the number of tasks $K$  becomes large, this may be too computationally expensive as it requires one to perform $K$ backward passes through the shared backbone to compute the $K$ gradients. Prior work suggested using a subset of tasks~\cite{liu2021conflict} or replacing the task gradients with the feature level gradient~\cite{sener2018multi,liu2020towards,javaloy2021rotograd} as potential practical speedups. We emphasize that this issue is not unique to our method, but rather is shared to all methods that compute all gradients for all tasks.

In practice, we found that using feature-level gradients as a surrogate to the gradient of the shared parameters dramatically degrades the performance of our method. See Appendix~\ref{sec:Appendix_speedup_feature_level} for empirical results and further discussion. As an alternative, we suggest updating the gradient weights $\alpha^{(t)}$ once every few iterations instead of every iteration. This simple yet effective solution greatly reduces the runtime (up to $\sim\times 10$ for QM9 and $\sim\times 5$ for MT10) while maintaining high performance. In Section~\ref{sec:scaling_up_nash_mtl}
we provide experimental results while varying the frequency of task weights update on the QM9 dataset and the MT10 benchmark. Our results show that \ourmethod{} runtime can be reduced to about the same as linear scalarization (or STL) while maintaining competitive results compared to other baselines; However, in some cases, we do see a noticeable drop in performance compared with our standard approach. 

\section{Related Work}
\label{sec:Related}
In multitask learning (MTL), one simultaneously solves several learning problems while sharing information among tasks~\cite{caruana1997multitask, ruder2017overview}, commonly through a joint hidden representation~\cite{zhang2014facial,dai2016instance,pinto2017learning,zhao2018modulation,Liu2019EndToEndML}.
\begin{figure}[!t]
    \centering
    \includegraphics[width=1.\linewidth, clip]{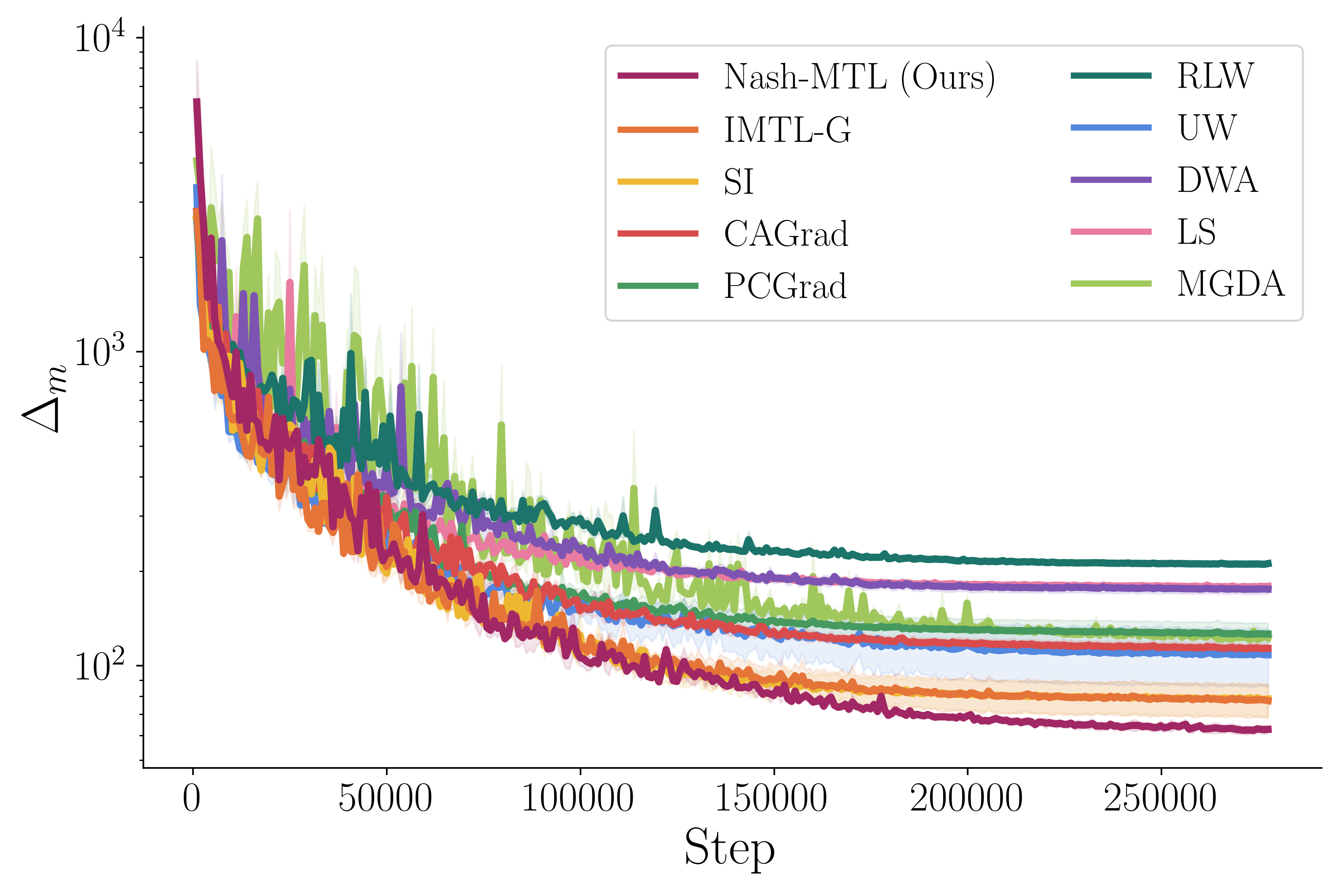}
    \caption{\textit{QM9}. Test $\Delta_m$ throughout the training process averaged over 3 random seeds.}
    \label{fig:qm_9}
\end{figure}
Studies in the literature proposed several explanations for the difficulty in the optimization process of MTL, such as conflicting gradients \cite{wang2020gradient, yu2020gradient}, or plateaus in the loss landscape \cite{schaul2019ray}. Other studies aimed at improving multitask learning by proposing novel architectures \cite{misra2016cross, hashimoto2017joint,Liu2019EndToEndML, Chen2020JustPA}. We focus on weighting the gradients of the tasks via an axiomatic approach that is agnostic to the architecture used. Studies in a similar vein proposed to weigh the task losses with various approaches, such as the uncertainty of the tasks \cite{kendall2018multi}, the norm of the gradients \cite{chen2017gradnorm}, random weights \cite{lin2021closer}, and similarity of the gradients \cite{du2018adapting, suteu2019regularizing}. These methods are mostly heuristic and can have unstable performance \cite{liu2021conflict}. 
Recently, several studies proposed MTL approaches based on the multiple-gradient descent algorithm (MGDA) for multi-objective optimization \cite{desideri2012multiple}. This is an appealing approach since, under mild conditions, convergence to a Pareto stationary point is guaranteed. \citet{sener2018multi} cast the multi-objective problem to multi-task problem and suggest task weighting based on the Frank-Wolfe algorithm \cite{jaggi2013revisiting}. \citet{liu2021conflict} searches for an update direction in a neighborhood of the average gradient that maximizes the worst improvement of any task. Unlike these studies, we propose an MTL approach based on a Bargaining game that can find solutions that are Pareto optimal and proportionally fair.

The closest work to our approach, to the best of our knowledge, is \citet{liu2020towards}. There, the authors propose to look for a fair gradient direction where all the cosine similarities are equal. We note that this update direction satisfies all of the Nash axioms except for Pareto optimally. Thus, unlike our proportionally fair approach, it can settle for a sub-optimal solution for the sake of fairness.

Finally, we note that the Nash bargaining solution was effectively applied to problems in various fields such as communication \cite{zhang2008cooperation, leshem2011smart, shi2018nash}, economics \cite{dagan1993bankruptcy}, and computing \cite{grosu2002load}, and to several learning setups, such as reinforcement learning \cite{qiao2006multi}, Bayesian optimization \cite{binois2020kalai}, clustering \cite{rezaee2021gbk}, federated learning \cite{kim2021cooperative}, and multi-armed bandits \cite{baek2021fair}.

\begin{table}[t]
\small
\centering
\caption{\textit{QM9}. Test performance averaged over 3 random seeds.}
    \vskip 0.11in
\begin{tabular}{@{}clccc@{}}
\toprule
 &  & \textbf{MR} $\downarrow$ & $\mathbf{\Delta_m\%}$ $\downarrow$ \\ 
 \midrule
\multicolumn{2}{c}{LS}   & 6.8 & $177.6 \pm ~~3.4$ \\
\multicolumn{2}{c}{SI}   & 4.0 & $~~77.8 \pm ~~9.2$ \\
\multicolumn{2}{c}{RLW}  & 8.2 & $203.8 \pm ~~3.4$ \\
\multicolumn{2}{c}{DWA}  & 6.4 & $175.3 \pm ~~6.3$ \\
\multicolumn{2}{c}{UW}   & 5.3 & $108.0 \pm 22.5$\\
\multicolumn{2}{c}{MGDA} & 5.9 & $120.5 \pm ~~2.0$\\
\multicolumn{2}{c}{PCGrad} & 5.0 &$125.7 \pm 10.3$ \\
\multicolumn{2}{c}{CAGrad} & 5.7 &$112.8 \pm ~~4.0$ \\ 
\multicolumn{2}{c}{IMTL-G} & 4.7 &$~~77.2 \pm ~~9.3$ \\ 
\midrule
\multicolumn{2}{c}{\ourmethod{}} & $\mathbf{2.5}$ & $\mathbf{~62.0 \pm ~1.4}$ \\ \bottomrule
\end{tabular}
\label{tab:qm9}
\end{table}

\begin{table*}[!t]
\setlength{\tabcolsep}{3pt}
\small
    \centering
    \caption{\textit{NYUv2}. Test performance for three tasks: semantic segmentation, depth estimation, and surface normal. Values are averages over 3 random seeds.}
    \vskip 0.11in
\resizebox{0.95\textwidth}{!}{%
\begin{tabular}{ccccccccccccccccccccc}
\toprule\\
 &  &  & \multicolumn{2}{c}{Segmentation} &  & \multicolumn{2}{c}{Depth} &  & \multicolumn{8}{c}{Surface Normal} &  &  &  \\
 \cmidrule(lr){4-5} \cmidrule(lr){7-8} \cmidrule(lr){10-17}
 &  &  & \multirow{2}{*}{mIoU $\uparrow$} & \multirow{2}{*}{Pix Acc $\uparrow$} &  & \multirow{2}{*}{Abs Err $\downarrow$} & \multirow{2}{*}{Rel Err $\downarrow$} &  & \multicolumn{2}{c}{Angle Distance $\downarrow$} &  & \multicolumn{5}{c}{Within $t^\circ$  $\uparrow$} &  & \textbf{MR} $\downarrow$ & $\mathbf{\Delta m \%} \downarrow$ &  \\
 \cmidrule(lr){10-11} \cmidrule(lr){13-17} \cmidrule(lr){19-19} \cmidrule(lr){20-20}
 &  &  &  &  &  &  &  &  & Mean & Median &  & 11.25 &  & 22.5 &  & 30 &  &  &  \\
 \midrule
 & \multicolumn{2}{c}{STL} & $38.30$ & $63.76$ &  & $0.6754$ & $0.2780$ &  & $25.01$ & $19.21$ &  & $30.14$ &  & $57.20$ &  & $69.15$ &  &  &  \\
  \midrule
 & \multicolumn{2}{c}{LS} & $39.29$ & $65.33$ &  & $0.5493$ & $0.2263$ &  & $28.15$ & $23.96$ &  & $22.09$ &  & $47.50$ &  & $61.08$ && ~~$8.11$ & $5.59$ &  \\
 & \multicolumn{2}{c}{SI} & $38.45$ & $64.27$ && $0.5354$ & $0.2201$ && $27.60$ & $23.37$ && $22.53$ && $48.57$ && $62.32$ && ~~$7.11$ & $4.39$  \\
 & \multicolumn{2}{c}{RLW} & $ 37.17 $ & $ 63.77 $ && $ 0.5759 $ & $ 0.2410 $ && $ 28.27 $ & $ 24.18 $ && $ 22.26 $ && $ 47.05 $ && $ 60.62 $ && $ 10.11 $ & $ 7.78 $  \\
 & \multicolumn{2}{c}{DWA} & $39.11$ & $65.31$ &  & $0.5510$ & $0.2285$ &  & $27.61$ & $23.18$ &  & $24.17$ &  & $50.18$ &  & $62.39$ && $6.88$ & $3.57$ &  \\
 & \multicolumn{2}{c}{UW} & $36.87$ & $63.17$ &  & $0.5446$ & $0.2260$ &  & $27.04$ & $22.61$ &  & $23.54$ &  & $49.05$ &  & $63.65$ && $6.44$ & $4.05$ &  \\
 & \multicolumn{2}{c}{MGDA} & $30.47$ & $59.90$ &  & $0.6070$ & $0.2555$ &  & $\mathbf{24.88}$ & $\mathbf{19.45}$ &  & $\mathbf{29.18}$ &  & $\mathbf{56.88}$ &  & $\mathbf{69.36}$ && $5.44$ & $1.38$ &  \\
 & \multicolumn{2}{c}{PCGrad} & $38.06$ & $64.64$ &  & $0.5550$ & $0.2325$ &  & $27.41$ & $22.80$ &  & $23.86$ &  & $49.83$ &  & $63.14$ && $6.88$ & $3.97$ &  \\
 & \multicolumn{2}{c}{GradDrop} & $39.39$ & $65.12$ &  & $0.5455$ & $0.2279$ &  & $27.48$ & $22.96$ &  & $23.38$ &  & $49.44$ &  & $62.87$ && $6.44$ & $3.58$ &  \\
 & \multicolumn{2}{c}{CAGrad} & $39.79$ & $65.49$ &  & $0.5486$ & $0.2250$ &  & $26.31$ & $21.58$ &  & $25.61$ &  & $52.36$ &  & $65.58$ && $3.77$ & $0.20$ &  \\
 & \multicolumn{2}{c}{IMTL-G} & $39.35$ & $ 65.60$ &  & $0.5426$ & $0.2256$ &  & $26.02$ & $21.19$ &  & ~~$26.2$ &  & $53.13$ &  & $66.24$ && $3.11 $ & $ -0.76 $ &  \\
 \midrule
 & \multicolumn{2}{c}{\ourmethod{}} & $\mathbf{40.13}$ & $\mathbf{65.93}$ &  & $\mathbf{0.5261}$ & $\mathbf{0.2171}$ &  & $25.26$ & $20.08$ &  & $28.4$ &  & $55.47$ &  & $68.15$ && $\mathbf{1.55}$ & $\mathbf{-4.04}$ & \\
 \bottomrule
\end{tabular}%
}
\label{tab:nyu}
\end{table*}

\section{Analysis}
\label{sec:Analysis}

\revision{We now analyze the convergence of our method in the convex and non-convex cases. As even single-task non-convex optimization might only converge to a stationary point, we will prove convergence to a Pareto stationary point, i.e., a point where some convex combination of the gradients is zero. As stated, we also assume that the gradients are independent while not at a Pareto stationary point. Independence of the gradients is a slightly stronger assumption than Pareto stationarity but is needed to exclude degenerate edge cases such as two identical tasks.}

\revision{We note that by substituting local Pareto optimality for Pareto stationarity in Assumption~\ref{assump:independence} we can show convergence to a local Pareto optimal point. However, this assumption has strong implications, as it implies we avoid local maxima and saddle points of any specific task. Since our update rule is a descent direction for all tasks, we can reasonably assume that our algorithm avoids local maxima points. Furthermore, it was shown that first-order methods avoid saddle points \cite{no_saddle_points}, giving credence to this stronger assumption. Nevertheless, we take a conservative approach and state our results with the weaker assumption.}


We formally make the following assumptions:
\begin{assumption}\label{assump:independence}
We assume that for a sequence $\{\theta^{(t)}\}_{t=1}^\infty$ generated by our algorithm, the set of the gradient vectors $g_1^{(t)},...,g_K^{(t)}$ at any point on the sequence and at any partial limit are linearly independent unless that point is a Pareto stationary point. 
\end{assumption}

\setlength{\tabcolsep}{5pt}
\begin{table*}[!t]
    \small
    \centering
    \caption{\textit{CityScapes}. Test performance for two tasks: semantic segmentation and depth estimation. Value are averages over 3 random seeds.}
    \vskip 0.11in
\begin{tabular}{@{}ccccccccc@{}}
\toprule
 & \multicolumn{2}{c}{Segmentation} & \multicolumn{1}{l}{} & \multicolumn{2}{c}{Depth} & \multicolumn{1}{l}{} &\\
\cmidrule(lr){2-3} \cmidrule(lr){5-6}
 & mIoU $\uparrow$ & \multicolumn{1}{l}{Pix Acc $\uparrow$} & \multicolumn{1}{l}{} & \multicolumn{1}{l}{Abs Err $\downarrow$} & \multicolumn{1}{l}{Rel Err$\downarrow$} & &
 \textbf{MR} $\downarrow$ & $\mathbf{\Delta_m\%}$ $\downarrow$\\
 \midrule
STL & 74.01 & 93.16 &  & 0.0125 & 27.77 & \multicolumn{1}{l}{} & \multicolumn{1}{l}{} \\ \midrule
LS & 75.18 & 93.49 &  & 0.0155 & 46.77 &  & $6.12$ & 22.60 \\
SI & 70.95 & 91.73 && 0.0161 & 33.83 && $8.00$ & 14.11 \\
RLW & 74.57 & 93.41 && 0.0158 & 47.79 && $9.25$ & 24.38 \\
DWA & 75.24 & 93.52 &  & 0.0160 & 44.37 && $6.00$ & 21.45 \\
UW & 72.02 & 92.85 &  & 0.0140 & \textbf{30.13} && $5.25$ & ~\textbf{~5.89} \\
MGDA & 68.84 & 91.54 &  & 0.0309 & 33.50 && $8.75$ & 44.14 \\
PCGrad & 75.13 & 93.48 &  & 0.0154 & 42.07 && $6.37$ & 18.29 \\
GradDrop & 75.27 & 93.53 &  & 0.0157 & 47.54 && $5.50$ & 23.73 \\
CAGrad & 75.16 & 93.48 &  & 0.0141 & 37.60 && $5.37$  & 11.64 \\ 
IMTL-G & 75.33 & 93.49 &  & 0.0135 & 38.41 && $3.62$  & 11.10 \\
\midrule
\ourmethod{} & \textbf{75.41} & \textbf{93.66} & \textbf{} & \textbf{0.0129} & 35.02 && $\mathbf{1.75}$ & 6.82 \\ \bottomrule
\end{tabular}
\label{tab:cityscapes}
\end{table*}

\begin{assumption}\label{assump:diff}
We assume that all loss functions are differentiable, bounded below and that all sub-level sets are bounded. The input domain is open and convex.
\end{assumption}

\begin{assumption}\label{assump:L-smooth}
We assume that all the loss functions are L-smooth,
\begin{equation}
    ||\nabla\ell_i(x)-\nabla\ell_i(y)||\leq L||x-y|| \quad .
\end{equation}
\end{assumption}

\begin{theorem}\label{trm:nonconvex}
Let $\{\theta^{(t)}\}_{t=1}^\infty$ be the sequence generated by the update rule $\theta^{(t+1)}=\theta^{(t)}-\mu^{(t)}\Dt^{(t)}$ where $\Dt^{(t)}=\sum_{i=1}^K\alpha^{(t)}_ig_i^{(t)}$ is the Nash bargaining solution $(G^{(t)})\T G^{(t)}\alpha^{(t)}=1/\alpha^{(t)}$. Set $\mu^{(t)}=\min\limits_{i\in[K]}\frac{1}{LK\alpha^{(t)}_i}$. Then, the sequence $\{\theta^{(t)}\}_{t=1}^\infty$ has a subsequence that converges to a Pareto stationary point $\theta^*$. Moreover all the loss functions $(\ell_1(\theta^{(t)}),...,\ell_K(\theta^{(t)}))$ converge to $(\ell_1(\theta^*),...,\ell_K(\theta^*))$.
\end{theorem}

\begin{proofsketch} We can show that $\mu^{(t)}=\min_i\frac{1}{\alpha^{(t)}_i}\rightarrow0$ so $||\alpha^{(t)}||\rightarrow\infty$. We also show that  $||1/\alpha^{(t)}||$ is bounded. As $(G^{(t)})\T G^{(t)}\alpha^{(t)}=1/\alpha^{(t)}$ this means that the smallest singluar value of $(G^{(t)})\T G^{(t)}$ must converge to zero. From compactness  $\{\theta^{(t)}\}_{t=1}^\infty$ has a converging subsequence whose limit we denote as $\theta^*$. From continuity we get that the gradients Gram matrix $G\T G$ computed at $\theta^*$ must have a zero singular value and therefore the gradients are linearly dependent. From our assumption this means that $\theta^*$ is Pareto stationary. As the losses are monotonically decearsing and bounded below they must converge and to the subsequence limit of $(\ell_1(\theta^*),...,\ell_K(\theta^*))$.
\end{proofsketch}
If we also assume convexity, we can strengthen our claim
\begin{theorem}
Let $\{\theta^{(t)}\}_{t=1}^\infty$ be the sequence generated by the update rule $\theta^{(t+1)}=\theta^{(t)}-\mu^{(t)}\Dt^{(t)}$ where $\Dt^{(t)}=\sum_{i=1}^K\alpha^{(t)}_ig_i^{(t)}$ is the Nash bargaining solution $(G^{(t)})\T G^{(t)}\alpha^{(t)}=1/\alpha^{(t)}$. Set $\mu^{(t)}=\min\limits_{i\in[K]}\frac{1}{LK\alpha^{(t)}_i}$. If we assume that all the loss functions are convex, then the sequence $\{\theta^{(t)}\}_{t=1}^\infty$ converges to a Pareto optimal point $\theta^*$.
\end{theorem}

See the full proofs in the appendix Sec. \ref{sec:Appenix_proofs}.

\section{Experiments}
\label{sec:Exp}

We evaluate \ourmethod{} on diverse multi-task learning problems. The experiments show the superiority of \ourmethod{} over previous MTL methods. To support future research and the reproducibility of the results, we will make our source code publicly available. Additional experimental results and details are provided in Appendix~\ref{sec:Appenix_exp_details}. 

\textbf{Compared methods:} We compare the following approaches: \textbf{(1)} Our proposed \ourmethod{} algorithm described in Section~\ref{sec:Method}; 
\textbf{(2)} Single task learning (STL), training an independent model for each task;
\textbf{(3)} Linear scalarization (LS) baseline which minimizes $\sum_k \ell_k$; 
\textbf{(4)} Scale-invariant (SI) baseline which minimizes $\sum_k \log\ell_k$. This baseline is invariant to rescaling each loss with a positive number;
\textbf{(5)} Dynamic Weight Average (DWA)~\cite{Liu2019EndToEndML} adjusts task weights based on the rates of loss changes over time; 
\textbf{(6)} Uncertainty weighting (UW)~\cite{kendall2018multi} uses task uncertainty quantification to adjust task weights;
\textbf{(7)} MGDA~\cite{sener2018multi} finds a convex combination of gradients with a minimal norm;
\textbf{(8)} Random loss weighting (RLW) with normal distribution, scales the losses according to randomly sampled task weights~\cite{lin2021closer}; \textbf{(9)} PCGrad~\cite{yu2020gradient} removes conflicting components of each gradient w.r.t the other gradients; 
\textbf{(10)} GradDrop~\cite{Chen2020JustPA} randomly drops components of the task gradients based on how much they conflict; 
\textbf{(11)} CAGrad~\cite{liu2021conflict} optimizes for the average loss while explicitly controlling the minimum decrease rate across tasks;
\textbf{(12)} IMTL-G~\cite{liu2020towards} uses an update direction with equal projections on task gradients. IMTL-G is applied to the feature-level gradients, as was suggested by the authors. We also tried applying IMTL-G to the shared-parameters gradient for a fair comparison, but its {performance was even worse}. 

\textbf{Evaluation.} For each experiment, we report the common evaluation metrics for each task. Since naturally MTL does not carry a single objective and since the scale of per-task metrics often varies significantly, we report two metrics that capture the overall performance: \textbf{(1)} $\mathbf{\Delta_m\%}$, the average per-task performance drop of method $m$ relative to the STL baseline denoted $b$. Formally,  $\Delta_m\%=\frac{1}{K}\sum_{k=1}^K (-1)^{\delta_k} (M_{m,k}-M_{b,k}) / M_{b,k}$,
where $M_{b,k}$ is the value of metric $M_k$ obtained by the baseline and $M_{m,k}$ by the compared method. $\delta_k=1$ if a higher value is better for a metric $M_k$ and 0 otherwise~\cite{maninis2019attentive, liu2021conflict}.
\textbf{(2) Mean Rank (MR):} The average rank of each method across the different tasks (lower is better). A method receives the best value, $\text{MR}=1$, if it ranks first in all tasks.

\subsection{Multi-Task Regression for QM9}
\label{sec:qm9}

We evaluate \ourmethod{} on predicting 11 properties of molecules from the QM9 dataset~\cite{ramakrishnan2014quantum}, a widely used benchmark for graph neural networks. QM9 consists of $\sim 130K$ molecules represented as graphs annotated with both node and edge features. We used the QM9 example in PyTorch Geometric~\cite{Fey/Lenssen/2019}, and use 110K molecules for training, 10K for validation, and 10K as a test set. As each task target range is at a different scale, this could be an issue for other methods that are not scale-invariant like ours. For fairness, we normalized each task target to have zero mean and unit standard deviation. We use the popular GNN model from \citet{gilmer2017neural}, a network comprised of several concatenated message passing layers, which update the node features based on both node and edge features, followed by the pooling operator from \citet{vinyals2015order}. Specifically, we used the implementation from \citet{Fey/Lenssen/2019}. We train each method for $300$ epochs and search for the best learning-rate (lr) given by the $\Delta_m$ performance on the validation set. We use a learning-rate scheduler to reduce the lr once the validation $\Delta_m$ metric has stopped improving. The validation set is also used for early stopping.

Predicting molecular properties in QM9 poses a significant  challenge for MTL methods because the number of tasks is large and because the loss scales vary significantly. The scale issue is only partially resolved by normalization because some tasks are easier to learn than others. Prior work found that single-task learning significantly improves performance on all targets compared to MTL methods~\cite{maron2019provably,Klicpera2020DirectionalMP}.

Results are shown in Figure~\ref{fig:qm_9} and Table~\ref{tab:qm9}. \ourmethod{} achieves the best performance in terms of both MR and $\Delta_m$. Interestingly, most MTL methods fall short compared to the simple scale-invariant baseline, which ignores gradient interaction, except for IMTL-G whose performance is on par with this baseline. This result shows that the scale-invariant property of our approach can be beneficial. See Appendix~\ref{sec:Appendix_qm9} for the per-task evaluation results. 

\subsection{Scene Understanding}
We follow the protocol of~\cite{Liu2019EndToEndML} and evaluate \ourmethod{} on the NYUv2 and Cityscapes datasets \citep{silberman2012indoor, Cordts2016Cityscapes}. NYUv2 is an indoor scene dataset that consists of 1449 RGBD images and dense per-pixel labeling with 13 classes. We use this dataset as a multitask learning benchmark for semantic segmentation, depth estimation, and surface normal prediction. 

The CityScapes dataset~\cite{Cordts2016Cityscapes} contains 5000 high-resolution street-view images with dense per-pixel annotations. We use this dataset as a multitask learning benchmark for semantic segmentation and depth estimation. To speed up the training phase, all images were resized to $128\times 256$. The original dataset contains 19 categories for pixel-wise semantic segmentation, together with ground-truth depth maps. For segmentation, we used a coarser version of the labels with 7 classes.

For all MTL methods, we train a Multi-Task Attention Network (MTAN)~\citep{Liu2019EndToEndML} 
which adds an attention mechanism on top of the SegNet architecture~\cite{badrinarayanan2017segnet}.
We follow the training procedure from~\citet{Liu2019EndToEndML,yu2020gradient,liu2021conflict}. Each method is trained for $200$ epochs with the Adam optimizer~\citep{Kingma2015AdamAM} and an initial learning-rate of $1e-4$. The learning-rate is halved to $5e-5$ after $100$ epochs. As in~\cite{liu2021conflict} The STL baseline refers to training task-specific SegNet models.

The results are presented in Table~\ref{tab:nyu} and Table~\ref{tab:cityscapes}. Our method, Nash-MTL, achieves the best MR in both datasets, the best $\Delta_m$ in NYUv2 and the seconds to best $\Delta_m$ in the CityScapes experiment. \ourmethod{} performance is well balanced across tasks. MGDA is primarily focused on the task of predicting surface normals and achieves poor performance on the other two tasks. The inherent biasedness of MGDA towards the task with the smallest gradient magnitude was previously discussed in \citet{liu2020towards}. \revision{We note that the optimal solution under \ourmethod{} for the two tasks case is equivalent to independently normalizing each gradient and summing with equal weights. While this is a fairly simple approach for MTL, we show that it outperforms almost all the compared MTL methods on the two-tasks CityScapes benchmark.}

\subsection{Multi-Task Reinforcement Learning}
\label{sec:mt10}
We consider a multi-task RL problem and evaluate \ourmethod{} on the MT10 environment from the Meta-World benchmark~\cite{yu2020meta}. This benchmark involves a simulated robot trained to perform actions like pressing a button and opening a window, each action treated as a task, for a total of 10 tasks. The goal is to learn a policy that can succeed across all the diverse sets of manipulation tasks.
\begin{table}[t]
\centering
\small
\caption{\textit{MT10}. Average success over 10 random seeds.}
\vskip 0.11in

\begin{tabular}{@{}cccc@{}}
\toprule
                &                & Success $\pm$ SEM          \\ \midrule
\multicolumn{2}{c}{STL SAC}      & $0.90 \pm 0.032$ \\
\midrule
\multicolumn{2}{c}{MTL SAC}      & $0.49 \pm 0.073$ \\
\multicolumn{2}{c}{MTL SAC + TE} & $0.54 \pm 0.047$ \\
\multicolumn{2}{c}{MH SAC}       & $0.61 \pm 0.036$ \\
\multicolumn{2}{c}{SM}           & $0.73 \pm 0.043$ \\
\multicolumn{2}{c}{CARE}         & $0.84 \pm 0.051$ \\
\multicolumn{2}{c}{PCGrad}       & $0.72 \pm 0.022$ \\
\multicolumn{2}{c}{CAGrad}       & $0.83 \pm 0.045$ \\
\midrule
\multicolumn{2}{c}{\ourmethod{}}         & $\mathbf{0.91 \pm 0.031}$ \\
\bottomrule
\end{tabular}
\label{tab:mt10}
\end{table}
Following previous works on MTL-RL~\cite{yu2020gradient, liu2021conflict,sodhani2021multi}, we use Soft Actor-Critic (SAC)~\cite{haarnoja2018soft} as the base RL algorithm. Along with the MTL methods (1) CAGrad~\cite{liu2021conflict} and (2) PCGrad~\cite{yu2020gradient} applied to a shared model SAC, we evaluate the following methods: (3) STL, one SAC model per task; (4) MTL SAC with a shared model; (5) Multi-task SAC with task encoder (MTL SAC + TE, ~\citet{yu2020meta}); (6) Multi-headed SAC (MH SAC) with task-specific heads~\cite{yu2020meta}; (7) Soft Modularization~(SM, \citet{Yang2020MultiTaskRL}) which estimates per-task routes for different tasks in a shared model, and; (8) CARE~\cite{sodhani2021multi} which utilizes language metadata and employs a mixture of encoders. We follow the same experiment setup
from~\citet{sodhani2021multi,liu2021conflict} to train all methods over 2 million steps and report the mean success over $10$ random seeds with fixed evaluation frequency. The results are presented in Table~\ref{tab:mt10}.

\ourmethod{} achieves the best performance by a large margin. In addition, \ourmethod{} is the only MTL method to reach the same performance as the per-task SAC STL baseline.

\subsection{Scaling-up Nash-MTL}
\label{sec:scaling_up_nash_mtl}

\revision{One of the major drawbacks of the SOTA MTL methods is that they require access to all task gradients to compute the optimal update direction~\cite{sener2018multi,yu2020gradient,liu2020towards,liu2021conflict}. This requires one to perform $K$ backward passes at each optimization step, thus scales poorly with the number of tasks. Previous works suggested using a subset of tasks~\cite{liu2021conflict} or replacing the task gradients with the feature-level gradient~\cite{sener2018multi,liu2020towards,javaloy2021rotograd} as potential speedups. 
In our experiments, we found that using the feature-level gradients can greatly reduce \ourmethod{} performance (Appendix~\ref{sec:Appendix_speedup_feature_level}). However, here we show that the simple solution of updating task weights less frequently maintains good performance while dramatically reducing the training time.}

\begin{table}
	    \vspace{-7pt}
		\caption{Training runtime per episode and 
		average success for the MT10 benchmark, computed over $10$ random seeds while varying the frequency of task weights updates in \ourmethod{}.}
		\label{tab:mt10_speedup}
		\centering
		\vspace{10pt}
		\begin{tabular}{@{}cccccc@{}}
        \toprule
                &                & Success $\pm$ SEM  & Runtime[Sec.]        \\ \midrule
        \multicolumn{2}{c}{MTL-SAC}      & $0.49 \pm 0.073$ & 7.3 \\
        \multicolumn{2}{c}{PCGrad}            & $0.72 \pm 0.022$ & 9.7\\
        \multicolumn{2}{c}{CAGrad}            & $0.83 \pm 0.045$ & 20.9\\
        \midrule
        \multicolumn{2}{c}{\ourmethod{}}      & $0.91 \pm 0.031$ & 40.7 \\
        \multicolumn{2}{c}{\ourmethod{}-50}      & $0.85 \pm 0.022$ & 8.6 \\
        \multicolumn{2}{c}{\ourmethod{}-100}      & $0.87 \pm 0.033$ & 7.9 \\
        \bottomrule
        \end{tabular}

\end{table}

\revision{One approach to alleviate this issue is to update the task weights less frequently, and use these weights in subsequent steps. We evaluate this approach using the QM9 dataset and the MT10 benchmark and present the result in Figure~\ref{fig:speedup_ablation} and Table~\ref{tab:mt10_speedup}. We denote \ourmethod{} with task weight update every $T$ optimization steps with \ourmethod{}-$T$.}

\revision{The results show that \ourmethod{} is fairly robust to varying intervals between weights updates. While this simple approach results in a small degradation in performance, it can dramatically decrease the training time of our method. For example, on the QM9, updating the weights every $5$/$50$ steps results in a $\times3.7/9.8$ speedup w.r.t updating the weights at each step. On the MT10 environment, updating the weights every $100$ steps result in $\sim\times 10$ speedup (only $\sim\times1.1$ slower than the fastest baseline) while outperforming all other MTL baseline method (Table~\ref{tab:mt10_speedup}).}

\begin{figure}
    \centering
		\vspace{0.in}
		\includegraphics[width=1.\linewidth]{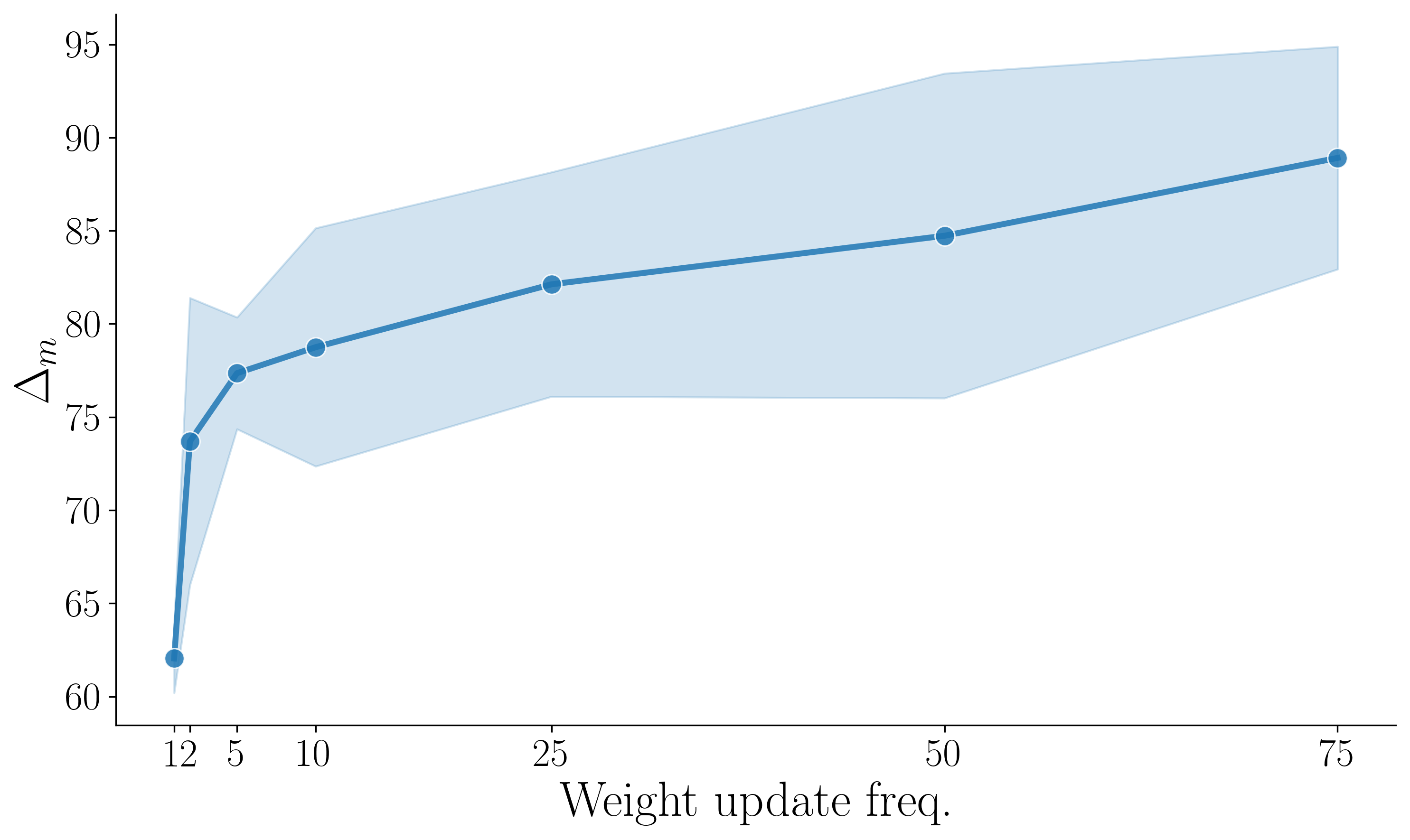}
		\vspace{-20pt}
		\captionof{figure}{Test $\Delta_m$ for the QM9 dataset, averaged over 3 random seeds, for different intervals of task weights update.}
		\label{fig:speedup_ablation}

\end{figure}

\section{Conclusion}

In this work, we present Nash-MTL, a novel and principled approach for multitask learning. We frame the gradient combination step in MTL as a bargaining game and use the Nash bargaining solution to find the optimal update direction. We highlight the importance of the scale invariance approach for multitask learning, specifically for setups with varying loss scales and gradient magnitudes. We provide a theoretical convergence analysis for \ourmethod{}, \revision{showing that it converges to a Pareto optimal and Pareto stationary points in the convex and non-convex settings, respectively.} Finally, our experiments show that \ourmethod{} achieves state-of-the-art results on various benchmarks across multiple domains.

\section{Acknowledgements}

This work was funded by the Israeli innovation authority through the AVATAR consortium; by the Israel
Science Foundation (ISF grant 737/2018); and by an equipment grant to GC and Bar Ilan University (ISF grant 2332/18).


\bibliography{ref}

\begin{thebibliography}{63}
\providecommand{\natexlab}[1]{#1}
\providecommand{\url}[1]{\texttt{#1}}
\expandafter\ifx\csname urlstyle\endcsname\relax
  \providecommand{\doi}[1]{doi: #1}\else
  \providecommand{\doi}{doi: \begingroup \urlstyle{rm}\Url}\fi

\bibitem[Achituve et~al.(2021)Achituve, Maron, and Chechik]{achituve2021self}
Achituve, I., Maron, H., and Chechik, G.
\newblock Self-supervised learning for domain adaptation on point clouds.
\newblock In \emph{Proceedings of the IEEE/CVF Winter Conference on
  Applications of Computer Vision}, pp.\  123--133, 2021.

\bibitem[Badrinarayanan et~al.(2017)Badrinarayanan, Kendall, and
  Cipolla]{badrinarayanan2017segnet}
Badrinarayanan, V., Kendall, A., and Cipolla, R.
\newblock Segnet: A deep convolutional encoder-decoder architecture for image
  segmentation.
\newblock \emph{IEEE transactions on pattern analysis and machine
  intelligence}, 39\penalty0 (12):\penalty0 2481--2495, 2017.

\bibitem[Baek \& Farias(2021)Baek and Farias]{baek2021fair}
Baek, J. and Farias, V.~F.
\newblock Fair exploration via axiomatic bargaining.
\newblock \emph{arXiv preprint arXiv:2106.02553}, 2021.

\bibitem[Baxter(2000)]{Baxter00}
Baxter, J.
\newblock A model of inductive bias learning.
\newblock \emph{J. Artif. Intell. Res.}, 2000.

\bibitem[Binois et~al.(2020)Binois, Picheny, Taillandier, and
  Habbal]{binois2020kalai}
Binois, M., Picheny, V., Taillandier, P., and Habbal, A.
\newblock The {K}alai-{S}morodinsky solution for many-objective {B}ayesian
  optimization.
\newblock \emph{J. Mach. Learn. Res.}, 21\penalty0 (150):\penalty0 1--42, 2020.

\bibitem[Caruana(1997)]{caruana1997multitask}
Caruana, R.
\newblock Multitask learning.
\newblock \emph{Machine learning}, 28\penalty0 (1):\penalty0 41--75, 1997.

\bibitem[Chen et~al.(2018)Chen, Badrinarayanan, Lee, and
  Rabinovich]{chen2017gradnorm}
Chen, Z., Badrinarayanan, V., Lee, C.-Y., and Rabinovich, A.
\newblock Gradnorm: Gradient normalization for adaptive loss balancing in deep
  multitask networks.
\newblock In \emph{International Conference on Machine Learning}, pp.\
  794--803. PMLR, 2018.

\bibitem[Chen et~al.(2020)Chen, Ngiam, Huang, Luong, Kretzschmar, Chai, and
  Anguelov]{Chen2020JustPA}
Chen, Z., Ngiam, J., Huang, Y., Luong, T., Kretzschmar, H., Chai, Y., and
  Anguelov, D.
\newblock Just pick a sign: Optimizing deep multitask models with gradient sign
  dropout.
\newblock \emph{ArXiv}, abs/2010.06808, 2020.

\bibitem[Cordts et~al.(2016)Cordts, Omran, Ramos, Rehfeld, Enzweiler, Benenson,
  Franke, Roth, and Schiele]{Cordts2016Cityscapes}
Cordts, M., Omran, M., Ramos, S., Rehfeld, T., Enzweiler, M., Benenson, R.,
  Franke, U., Roth, S., and Schiele, B.
\newblock The cityscapes dataset for semantic urban scene understanding.
\newblock In \emph{Proc. of the IEEE Conference on Computer Vision and Pattern
  Recognition (CVPR)}, 2016.

\bibitem[Crawshaw(2020)]{crawshaw2020multi}
Crawshaw, M.
\newblock Multi-task learning with deep neural networks: A survey.
\newblock \emph{arXiv preprint arXiv:2009.09796}, 2020.

\bibitem[Dagan \& Volij(1993)Dagan and Volij]{dagan1993bankruptcy}
Dagan, N. and Volij, O.
\newblock The bankruptcy problem: a cooperative bargaining approach.
\newblock \emph{Mathematical Social Sciences}, 26\penalty0 (3):\penalty0
  287--297, 1993.

\bibitem[Dai et~al.(2016)Dai, He, and Sun]{dai2016instance}
Dai, J., He, K., and Sun, J.
\newblock Instance-aware semantic segmentation via multi-task network cascades.
\newblock In \emph{Proceedings of the IEEE conference on computer vision and
  pattern recognition}, pp.\  3150--3158, 2016.

\bibitem[D{\'e}sid{\'e}ri(2012)]{desideri2012multiple}
D{\'e}sid{\'e}ri, J.-A.
\newblock Multiple-gradient descent algorithm ({MGDA}) for multiobjective
  optimization.
\newblock \emph{Comptes Rendus Mathematique}, 350\penalty0 (5-6):\penalty0
  313--318, 2012.

\bibitem[Du et~al.(2018)Du, Czarnecki, Jayakumar, Farajtabar, Pascanu, and
  Lakshminarayanan]{du2018adapting}
Du, Y., Czarnecki, W.~M., Jayakumar, S.~M., Farajtabar, M., Pascanu, R., and
  Lakshminarayanan, B.
\newblock Adapting auxiliary losses using gradient similarity.
\newblock \emph{arXiv preprint arXiv:1812.02224}, 2018.

\bibitem[Fey \& Lenssen(2019)Fey and Lenssen]{Fey/Lenssen/2019}
Fey, M. and Lenssen, J.~E.
\newblock Fast graph representation learning with {PyTorch Geometric}.
\newblock In \emph{ICLR Workshop on Representation Learning on Graphs and
  Manifolds}, 2019.

\bibitem[Gilmer et~al.(2017)Gilmer, Schoenholz, Riley, Vinyals, and
  Dahl]{gilmer2017neural}
Gilmer, J., Schoenholz, S.~S., Riley, P.~F., Vinyals, O., and Dahl, G.~E.
\newblock Neural message passing for quantum chemistry.
\newblock In \emph{International conference on machine learning}, pp.\
  1263--1272. PMLR, 2017.

\bibitem[Grosu et~al.(2002)Grosu, Chronopoulos, and Leung]{grosu2002load}
Grosu, D., Chronopoulos, A.~T., and Leung, M.-Y.
\newblock Load balancing in distributed systems: An approach using cooperative
  games.
\newblock In \emph{Proceedings 16th International Parallel and Distributed
  Processing Symposium}, pp.\  10--pp. IEEE, 2002.

\bibitem[Haarnoja et~al.(2018)Haarnoja, Zhou, Abbeel, and
  Levine]{haarnoja2018soft}
Haarnoja, T., Zhou, A., Abbeel, P., and Levine, S.
\newblock Soft actor-critic: Off-policy maximum entropy deep reinforcement,
  2018.

\bibitem[Hashimoto et~al.(2017)Hashimoto, Xiong, Tsuruoka, and
  Socher]{hashimoto2017joint}
Hashimoto, K., Xiong, C., Tsuruoka, Y., and Socher, R.
\newblock A joint many-task model: Growing a neural network for multiple nlp
  tasks.
\newblock In \emph{Proceedings of the 2017 Conference on Empirical Methods in
  Natural Language Processing}, pp.\  1923--1933, 2017.

\bibitem[Jaggi(2013)]{jaggi2013revisiting}
Jaggi, M.
\newblock Revisiting {F}rank-{W}olfe: Projection-free sparse convex
  optimization.
\newblock In \emph{International Conference on Machine Learning}, pp.\
  427--435. PMLR, 2013.

\bibitem[Javaloy \& Valera(2021)Javaloy and Valera]{javaloy2021rotograd}
Javaloy, A. and Valera, I.
\newblock Rotograd: Dynamic gradient homogenization for multi-task learning.
\newblock \emph{arXiv preprint arXiv:2103.02631}, 2021.

\bibitem[Kendall et~al.(2018)Kendall, Gal, and Cipolla]{kendall2018multi}
Kendall, A., Gal, Y., and Cipolla, R.
\newblock Multi-task learning using uncertainty to weigh losses for scene
  geometry and semantics.
\newblock In \emph{Proceedings of the IEEE conference on computer vision and
  pattern recognition}, pp.\  7482--7491, 2018.

\bibitem[Kim(2021)]{kim2021cooperative}
Kim, S.
\newblock Cooperative federated learning-based task offloading scheme for
  tactical edge networks.
\newblock \emph{IEEE Access}, 9:\penalty0 145739--145747, 2021.

\bibitem[Kingma \& Ba(2015)Kingma and Ba]{Kingma2015AdamAM}
Kingma, D.~P. and Ba, J.
\newblock Adam: A method for stochastic optimization.
\newblock \emph{CoRR}, abs/1412.6980, 2015.

\bibitem[Klicpera et~al.(2020)Klicpera, Gro{\ss}, and
  G{\"u}nnemann]{Klicpera2020DirectionalMP}
Klicpera, J., Gro{\ss}, J., and G{\"u}nnemann, S.
\newblock Directional message passing for molecular graphs.
\newblock \emph{ArXiv}, abs/2003.03123, 2020.

\bibitem[Leshem \& Zehavi(2011)Leshem and Zehavi]{leshem2011smart}
Leshem, A. and Zehavi, E.
\newblock Smart carrier sensing for distributed computation of the generalized
  nash bargaining solution.
\newblock In \emph{2011 17th International Conference on Digital Signal
  Processing (DSP)}, pp.\  1--5. IEEE, 2011.

\bibitem[Lin et~al.(2021)Lin, Ye, and Zhang]{lin2021closer}
Lin, B., Ye, F., and Zhang, Y.
\newblock A closer look at loss weighting in multi-task learning.
\newblock \emph{arXiv preprint arXiv:2111.10603}, 2021.

\bibitem[Lipp \& Boyd(2016)Lipp and Boyd]{lipp2016variations}
Lipp, T. and Boyd, S.
\newblock Variations and extension of the convex--concave procedure.
\newblock \emph{Optimization and Engineering}, 17\penalty0 (2):\penalty0
  263--287, 2016.

\bibitem[Liu et~al.(2021{\natexlab{a}})Liu, Liu, Jin, Stone, and
  Liu]{liu2021conflict}
Liu, B., Liu, X., Jin, X., Stone, P., and Liu, Q.
\newblock Conflict-averse gradient descent for multi-task learning.
\newblock \emph{Advances in Neural Information Processing Systems}, 34,
  2021{\natexlab{a}}.

\bibitem[Liu et~al.(2021{\natexlab{b}})Liu, Li, Kuang, Xue, Chen, Yang, Liao,
  and Zhang]{liu2020towards}
Liu, L., Li, Y., Kuang, Z., Xue, J.-H., Chen, Y., Yang, W., Liao, Q., and
  Zhang, W.
\newblock Towards impartial multi-task learning.
\newblock In \emph{International Conference on Learning Representations},
  2021{\natexlab{b}}.

\bibitem[Liu et~al.(2019{\natexlab{a}})Liu, Davison, and Johns]{liu2019self}
Liu, S., Davison, A., and Johns, E.
\newblock Self-supervised generalisation with meta auxiliary learning.
\newblock \emph{Advances in Neural Information Processing Systems}, 32,
  2019{\natexlab{a}}.

\bibitem[Liu et~al.(2019{\natexlab{b}})Liu, Johns, and
  Davison]{Liu2019EndToEndML}
Liu, S., Johns, E., and Davison, A.~J.
\newblock End-to-end multi-task learning with attention.
\newblock \emph{2019 IEEE/CVF Conference on Computer Vision and Pattern
  Recognition (CVPR)}, pp.\  1871--1880, 2019{\natexlab{b}}.

\bibitem[Maninis et~al.(2019)Maninis, Radosavovic, and
  Kokkinos]{maninis2019attentive}
Maninis, K.-K., Radosavovic, I., and Kokkinos, I.
\newblock Attentive single-tasking of multiple tasks.
\newblock In \emph{Proceedings of the IEEE/CVF Conference on Computer Vision
  and Pattern Recognition}, pp.\  1851--1860, 2019.

\bibitem[Maron et~al.(2019)Maron, Ben-Hamu, Serviansky, and
  Lipman]{maron2019provably}
Maron, H., Ben-Hamu, H., Serviansky, H., and Lipman, Y.
\newblock Provably powerful graph networks.
\newblock \emph{arXiv preprint arXiv:1905.11136}, 2019.

\bibitem[Misra et~al.(2016)Misra, Shrivastava, Gupta, and
  Hebert]{misra2016cross}
Misra, I., Shrivastava, A., Gupta, A., and Hebert, M.
\newblock Cross-stitch networks for multi-task learning.
\newblock In \emph{Proceedings of the IEEE conference on computer vision and
  pattern recognition}, pp.\  3994--4003, 2016.

\bibitem[Nash(1953)]{nash}
Nash, J.
\newblock Two-person cooperative games.
\newblock \emph{Econometrica}, 21\penalty0 (1):\penalty0 128--140, 1953.
\newblock ISSN 00129682, 14680262.
\newblock URL \url{http://www.jstor.org/stable/1906951}.

\bibitem[Navon et~al.(2021{\natexlab{a}})Navon, Achituve, Maron, Chechik, and
  Fetaya]{NavonAMCF21}
Navon, A., Achituve, I., Maron, H., Chechik, G., and Fetaya, E.
\newblock Auxiliary learning by implicit differentiation.
\newblock In \emph{International Conference on Learning Representations
  {(ICLR)}}, 2021{\natexlab{a}}.

\bibitem[Navon et~al.(2021{\natexlab{b}})Navon, Shamsian, Chechik, and
  Fetaya]{navon2021learning}
Navon, A., Shamsian, A., Chechik, G., and Fetaya, E.
\newblock Learning the pareto front with hypernetworks.
\newblock In \emph{International Conference on Learning Representations},
  2021{\natexlab{b}}.
\newblock URL \url{https://openreview.net/forum?id=NjF772F4ZZR}.

\bibitem[Panageas et~al.(2019)Panageas, Piliouras, and Wang]{no_saddle_points}
Panageas, I., Piliouras, G., and Wang, X.
\newblock First-order methods almost always avoid saddle points: The case of
  vanishing step-sizes.
\newblock In \emph{Neural Information Processing Systems (NeurIPS)}, 2019.

\bibitem[Pinto \& Gupta(2017)Pinto and Gupta]{pinto2017learning}
Pinto, L. and Gupta, A.
\newblock Learning to push by grasping: Using multiple tasks for effective
  learning.
\newblock In \emph{2017 IEEE international conference on robotics and
  automation (ICRA)}, pp.\  2161--2168. IEEE, 2017.

\bibitem[Qiao et~al.(2006)Qiao, Rozenblit, Szidarovszky, and
  Yang]{qiao2006multi}
Qiao, H., Rozenblit, J., Szidarovszky, F., and Yang, L.
\newblock Multi-agent learning model with bargaining.
\newblock In \emph{Proceedings of the 2006 winter simulation conference}, pp.\
  934--940. IEEE, 2006.

\bibitem[Ramakrishnan et~al.(2014)Ramakrishnan, Dral, Rupp, and
  Von~Lilienfeld]{ramakrishnan2014quantum}
Ramakrishnan, R., Dral, P.~O., Rupp, M., and Von~Lilienfeld, O.~A.
\newblock Quantum chemistry structures and properties of 134 kilo molecules.
\newblock \emph{Scientific data}, 1\penalty0 (1):\penalty0 1--7, 2014.

\bibitem[Rezaee et~al.(2021)Rezaee, Eshkevari, Saberi, and
  Hussain]{rezaee2021gbk}
Rezaee, M.~J., Eshkevari, M., Saberi, M., and Hussain, O.
\newblock {GBK}-means clustering algorithm: An improvement to the {K}-means
  algorithm based on the bargaining game.
\newblock \emph{Knowledge-Based Systems}, 213:\penalty0 106672, 2021.

\bibitem[Ruder(2017)]{ruder2017overview}
Ruder, S.
\newblock An overview of multi-task learning in deep neural networks.
\newblock \emph{arXiv preprint arXiv:1706.05098}, 2017.

\bibitem[Schaul et~al.(2019)Schaul, Borsa, Modayil, and Pascanu]{schaul2019ray}
Schaul, T., Borsa, D., Modayil, J., and Pascanu, R.
\newblock Ray interference: a source of plateaus in deep reinforcement
  learning.
\newblock \emph{arXiv preprint arXiv:1904.11455}, 2019.

\bibitem[Sener \& Koltun(2018)Sener and Koltun]{sener2018multi}
Sener, O. and Koltun, V.
\newblock Multi-task learning as multi-objective optimization.
\newblock In \emph{Advances in Neural Information Processing Systems}, pp.\
  527--538, 2018.

\bibitem[Shi et~al.(2018)Shi, Wang, Salous, Zhou, and Hu]{shi2018nash}
Shi, C., Wang, F., Salous, S., Zhou, J., and Hu, Z.
\newblock Nash bargaining game-theoretic framework for power control in
  distributed multiple-radar architecture underlying wireless communication
  system.
\newblock \emph{Entropy}, 20\penalty0 (4):\penalty0 267, 2018.

\bibitem[Silberman et~al.(2012)Silberman, Hoiem, Kohli, and
  Fergus]{silberman2012indoor}
Silberman, N., Hoiem, D., Kohli, P., and Fergus, R.
\newblock Indoor segmentation and support inference from rgbd images.
\newblock In \emph{European conference on computer vision}, pp.\  746--760.
  Springer, 2012.

\bibitem[Sodhani et~al.(2021)Sodhani, Zhang, and Pineau]{sodhani2021multi}
Sodhani, S., Zhang, A., and Pineau, J.
\newblock Multi-task reinforcement learning with context-based representations.
\newblock \emph{arXiv preprint arXiv:2102.06177}, 2021.

\bibitem[Sriperumbudur \& Lanckriet(2009)Sriperumbudur and
  Lanckriet]{sriperumbudur2009convergence}
Sriperumbudur, B.~K. and Lanckriet, G.~R.
\newblock On the convergence of the concave-convex procedure.
\newblock In \emph{Nips}, volume~9, pp.\  1759--1767. Citeseer, 2009.

\bibitem[Standley et~al.(2020)Standley, Zamir, Chen, Guibas, Malik, and
  Savarese]{tasks_together}
Standley, T., Zamir, A.~R., Chen, D., Guibas, L.~J., Malik, J., and Savarese,
  S.
\newblock Which tasks should be learned together in multi-task learning?
\newblock In \emph{International Conference on Machine Learning {ICML}}, 2020.

\bibitem[Suteu \& Guo(2019)Suteu and Guo]{suteu2019regularizing}
Suteu, M. and Guo, Y.
\newblock Regularizing deep multi-task networks using orthogonal gradients.
\newblock \emph{arXiv preprint arXiv:1912.06844}, 2019.

\bibitem[Szép \& Forgó(1985)Szép and Forgó]{game_theory}
Szép, J. and Forgó, F.
\newblock \emph{Introduction to the Theory of Games}.
\newblock Springer, 1985.

\bibitem[Thomson(1994)]{coop_barg}
Thomson, W.
\newblock Chapter 35 cooperative models of bargaining.
\newblock volume~2 of \emph{Handbook of Game Theory with Economic
  Applications}, pp.\  1237--1284. Elsevier, 1994.

\bibitem[Vinyals et~al.(2015)Vinyals, Bengio, and Kudlur]{vinyals2015order}
Vinyals, O., Bengio, S., and Kudlur, M.
\newblock Order matters: Sequence to sequence for sets.
\newblock \emph{arXiv preprint arXiv:1511.06391}, 2015.

\bibitem[Wang et~al.(2020)Wang, Tsvetkov, Firat, and Cao]{wang2020gradient}
Wang, Z., Tsvetkov, Y., Firat, O., and Cao, Y.
\newblock Gradient vaccine: Investigating and improving multi-task optimization
  in massively multilingual models.
\newblock In \emph{International Conference on Learning Representations}, 2020.

\bibitem[Yang et~al.(2020)Yang, Xu, Wu, and Wang]{Yang2020MultiTaskRL}
Yang, R., Xu, H., Wu, Y., and Wang, X.
\newblock Multi-task reinforcement learning with soft modularization.
\newblock \emph{ArXiv}, abs/2003.13661, 2020.

\bibitem[Yu et~al.(2020{\natexlab{a}})Yu, Kumar, Gupta, Levine, Hausman, and
  Finn]{yu2020gradient}
Yu, T., Kumar, S., Gupta, A., Levine, S., Hausman, K., and Finn, C.
\newblock Gradient surgery for multi-task learning.
\newblock In \emph{Advances in Neural Information Processing Systems},
  2020{\natexlab{a}}.

\bibitem[Yu et~al.(2020{\natexlab{b}})Yu, Quillen, He, Julian, Hausman, Finn,
  and Levine]{yu2020meta}
Yu, T., Quillen, D., He, Z., Julian, R., Hausman, K., Finn, C., and Levine, S.
\newblock Meta-world: A benchmark and evaluation for multi-task and meta
  reinforcement learning.
\newblock In \emph{Conference on Robot Learning}, pp.\  1094--1100. PMLR,
  2020{\natexlab{b}}.

\bibitem[Yuille \& Rangarajan(2003)Yuille and Rangarajan]{yuille2003concave}
Yuille, A.~L. and Rangarajan, A.
\newblock The concave-convex procedure.
\newblock \emph{Neural computation}, 15\penalty0 (4):\penalty0 915--936, 2003.

\bibitem[Zhang et~al.(2008)Zhang, Shi, Chen, Guizani, and
  Qiu]{zhang2008cooperation}
Zhang, Z., Shi, J., Chen, H.-H., Guizani, M., and Qiu, P.
\newblock A cooperation strategy based on nash bargaining solution in
  cooperative relay networks.
\newblock \emph{IEEE Transactions on Vehicular Technology}, 57\penalty0
  (4):\penalty0 2570--2577, 2008.

\bibitem[Zhang et~al.(2014)Zhang, Luo, Loy, and Tang]{zhang2014facial}
Zhang, Z., Luo, P., Loy, C.~C., and Tang, X.
\newblock Facial landmark detection by deep multi-task learning.
\newblock In \emph{European conference on computer vision}, pp.\  94--108.
  Springer, 2014.

\bibitem[Zhao et~al.(2018)Zhao, Li, Shen, Liang, and Wu]{zhao2018modulation}
Zhao, X., Li, H., Shen, X., Liang, X., and Wu, Y.
\newblock A modulation module for multi-task learning with applications in
  image retrieval.
\newblock In \emph{Proceedings of the European Conference on Computer Vision
  (ECCV)}, pp.\  401--416, 2018.

\end{thebibliography}
\bibliographystyle{icml2022}

\newpage
\appendix
\onecolumn
\section{Proofs}\label{sec:Appenix_proofs}
\begin{lemma}\label{Kenji_lemma}
If $\mathcal{L}$ is differential and L-smooth (assumption \ref{assump:L-smooth}) then 
$\mathcal{L}(\theta')\leq \mathcal{L}(\theta)+\nabla\mathcal{L}(\theta)\T (\theta'-\theta)+\frac{L}{2}\|\theta'-\theta\|^{2}$.
\end{lemma}
\begin{proof}
Fix $\theta,\theta'\in \dom(\mathcal{L}) \subseteq \RR^{d_{}}$. Since $\dom(\mathcal{L})$ is a convex and open set,  there exists $\epsilon >0$ such that $\theta+t(\theta'-\theta)\in\dom(\mathcal{L})$ for all $t\in[-\epsilon,1+\epsilon]$. Set $\epsilon>0$ to be such a number.
Thus, we can define a function $\bar \Lcal: [-\epsilon, 1+\epsilon] \rightarrow \RR$ by $\bar \Lcal(t)=\Lcal(\theta+t(\theta'-\theta))$. With this, $\bar \Lcal(1)=\Lcal(\theta')$, $\bar \Lcal(0)=\Lcal(\theta)$, and $\nabla \bar \Lcal(t)=\nabla\Lcal(\theta+t(\theta'-\theta))\T (\theta'-\theta)$ for $t \in [0, 1] \subset (-\epsilon,1+\epsilon)$. From Assumption \ref{assump:L-smooth}, $\|\nabla \Lcal(\theta') - \nabla \Lcal(\theta)\| \le L_{ } \|\theta'-\theta\|$, therefore 
\begin{align*}
\|\nabla\bar  \Lcal(t')-\nabla\bar  \Lcal(t) \| &=\|[\nabla\Lcal(\theta+t'(\theta'-\theta)) -\nabla\Lcal(\theta+t(\theta'-\theta))\T (\theta'-\theta) \|
\\ &\le \|\theta'-\theta\|\|\nabla\Lcal(\theta+t'(\theta'-\theta)) -\nabla\Lcal(\theta+t(\theta'-\theta))  \| \\ & \le  L_{ }\|\theta'-\theta\|\|(t'-t)(\theta'-\theta)   \|
\\ & \le L_{ }\|\theta'-\theta\|^{2}\|t'-t   \|.
\end{align*}
Hence, $\nabla \bar \Lcal:[0, 1]\rightarrow \RR$ is Lipschitz continuous, and therefore continuous. By using the fundamental theorem of calculus with the continuous function $\nabla\bar  \Lcal:[0, 1]  \rightarrow \RR$,
\begin{align} \label{eq:5}
\nonumber \mathcal{L}(\theta')&=\mathcal{L}(\theta)+ \int_0^1 \nabla\mathcal{L}(\theta+t(\theta'-\theta))\T (\theta'-\theta)dt
\\\nonumber &=\mathcal{L}(\theta)+\nabla\mathcal{L}(\theta)\T (\theta'-\theta)+ \int_0^1 \left(\nabla\mathcal{L}(\theta+t(\theta'-\theta))-\nabla\mathcal{L}(\theta)\right)\T (\theta'-\theta)dt
\\ \nonumber & \le \mathcal{L}(\theta)+\nabla\mathcal{L}(\theta)\T (\theta'-\theta)+ \int_0^1 \|\nabla\mathcal{L}(\theta+t(\theta'-\theta))-\nabla\mathcal{L}(\theta)\| \|\theta'-\theta \|dt
\\ \nonumber & \le \mathcal{L}(\theta)+\nabla\mathcal{L}(\theta)\T (\theta'-\theta)+ \int_0^1 t L_{ \mathcal{}}\|\theta'-\theta\|^{2}dt
\\ & =  \mathcal{L}(\theta)+\nabla\mathcal{L}(\theta)\T (\theta'-\theta)+\frac{L}{2}\|\theta'-\theta\|^{2}.
\end{align}

\end{proof}

\textbf{Theorem (5.4).} \textit{
Let $\{\theta^{(t)}\}_{t=1}^\infty$ be the sequence generated by the update rule $\theta^{(t+1)}=\theta^{(t)}-\mu^{(t)}\Dt^{(t)}$ where $\Dt^{(t)}=\sum_{i=1}^K\alpha^{(t)}_ig_i^{(t)}$ is the Nash bargaining solution $(G^{(t)})\T G^{(t)}\alpha^{(t)}=1/\alpha^{(t)}$. 
Set $\mu^{(t)}=\min\limits_{i\in[K]}\frac{1}{LK\alpha^{(t)}_i}$. 
The sequence $\{\theta^{(t)}\}_{t=1}^\infty$ has a subsequence that converges to a Pareto stationary point $\theta^*$. Moreover all the loss functions $(\ell_1(\theta^{(t)}),...,\ell_K(\theta^{(t)}))$ converge to $(\ell_1(\theta^*),...,\ell_K(\theta^*))$.
}


\begin{proof}
We first note that if for some step we reach a Pareto stationary solution the algorithm halts and sequence stays fixed at that point and therefore converges; Next, we assume that we never get to an exact Pareto stationary solution at any finite step.

We note that the norm of $\Dt^{(t)}$ is $\sqrt{K}$ as $||\Dt^{(t)}||^2=\sum_{i=1}^K\alpha_ig_i\T \Dt^{(t)}=\sum_{i=1}^K\alpha_i\cdot 1/\alpha_i=K$. 
For each loss $\ell_i$ we have using Lemma \ref{Kenji_lemma}

\begin{align}
    &\ell_i(\theta^{(t+1)})\leq \ell_i(\theta^{(t)})-  \mu^{(t)}\nabla\ell_i(\theta^{(t)})\T\Dt^{(t)}+\frac{L}{2}||\mu^{(t)}\Dt^{(t)}||^2=\\\label{eq:temp}
    &\ell_i(\theta^{(t)})- \mu^{(t)}\frac{1}{\alpha^{(t)}_i}+\frac{(\mu^{(t)})^2 LK}{2}\\
    &=\ell_i(\theta^{(t)})- \frac{\mu^{(t)}}{\alpha^{(t)}_i}+\frac{\mu^{(t)}}{2}\min_j\frac{1}{\alpha_j^{(t)}}\leq \ell_i(\theta^{(t)})- \frac{\mu^{(t)}}{2\alpha^{(t)}_i}<\ell_i(\theta^{(t)})\label{eq:b1}
\end{align}\\ 
This shows that our update decreases all the loss functions. 
We can average over inequality \ref{eq:temp} over all losses and get for $\mathcal{L}(\theta)=\frac{1}{K}\sum_{i=1}^K\ell_i(\theta)$:

\begin{align}
    &\mathcal{L}(\theta^{(t+1)})\leq \mathcal{L}(\theta^{(t)})- \mu^{(t)}\frac{1}{K}\sum_{i=1}^K\frac{1}{\alpha^{(t)}_i}+\frac{(\mu^{(t)})^2LK}{2}\leq
    \mathcal{L}(\theta^{(t)})-LK(\mu^{(t)})^2+\frac{(\mu^{(t)})^2LK}{2}
    =\mathcal{L}(\theta^{(t)})-\frac{LK(\mu^{(t)})^2}{2}. \label{eq:b2}
\end{align}


From this we can conclude that $\sum_{\tau=1}^t\frac{LK(\mu^{(\tau)})^2}{2}\leq \mathcal{L}(\theta_{1})-\mathcal{L}(\theta^{(t+1)})$. As $\mathcal{L}(\theta^{(t)})$ is bounded below we must have that the infinite series $\sum_{t=1}^\infty\frac{LK(\mu^{(t)})^2}{2}<\infty$, and also $\mu^{(t)}\to 0$. It follows that $\min_{i\in[K]} 1/\alpha^{(t)}_i\to0$ and therefore $||\alpha^{(t)}||\to\infty$.\\

We will now show that $||1/\alpha^{(t)}||$ is bounded for $t\to\infty$. As the sequence $\mathcal{L}(\theta^{(t)})$ is decreasing we have that the sequence $\theta^{(t)}$ is in the sublevel set $\{\theta:\mathcal{L}(\theta)\leq \mathcal{L}(\theta_0)\}$ which is closed and bounded and therefore compact. If follows that there exists $M<\infty$ such that $||g^{(t)}_i||\leq M$ for all $t$ and $i\in[K]$. We have for all $i$ and $t$, 
$|1/\alpha^{(t)}_i|=|(g^{(t)}_i)^T\theta^{(t)}|\leq \sqrt{K}||g^{(t)}_i||\leq \sqrt{K}M <\infty$, and so $||1/\alpha^{(t)}||$ is bounded. 
Combining these two results we have $||1/\alpha^{(t)}||\geq \sigma_K((G^{(t)})\T G^{(t)})||\alpha^{(t)}||$ where $\sigma_K((G^{(t)})\T G^{(t)})$ is the smallest singular value of $(G^{(t)})\T G^{(t)}$.
Since the norm of $\alpha^{(t)}$ goes to infinity and the norm $1/\alpha^{(t)}$ is bounded, it follows that $\sigma_K((G^{(t)})\T G^{(t)})\to 0$.

Now, since $\{\theta : \Lcal(\theta)\leq \Lcal{\theta_0}\}$ is compact there exists a subsequence $\theta^{(t_j)}$ that converges to some point $\theta^*$. 
As $\sigma_K((G^{(t)})^T G^{(t)})\rightarrow 0$ we have from continuity that $\sigma_K(G_*\T G_*)=0$ where $G_*$ is the matrix of gradients at $\theta^*$. This means that the gradients at $\theta$ are linearly dependent and therefore $\theta^*$ is Pareto stationary by assumption \ref{assump:independence}. As for all $i$ the sequence $\{\ell_i(\theta^{(t)})\}_{t=1}^\infty$ is monotonically decreasing and bounded below they all converges. Since $\ell_i(\theta^*)$ is the limit of a subsequence we get that $\ell_i(\theta^{(t)})\xrightarrow{t\rightarrow\infty}\ell_i(\theta^*)$.

\end{proof}

We now show that if we add a convexity assumption then we can prove convergence to the Pareto front.

\textbf{Theorem (5.5).} \textit{
Let $\{\theta^{(t)}\}_{t=1}^\infty$ be the sequence generated by the update rule $\theta^{(t+1)}=\theta^{(t)}-\mu^{(t)}\Dt^{(t)}$ where $\Dt^{(t)}=\sum_{i=1}^K\alpha^{(t)}_ig_i^{(t)}$ is the Nash bargaining solution $(G^{(t)})\T G^{(t)}\alpha^{(t)}=1/\alpha^{(t)}$. Set $\mu^{(t)}=\min\limits_{i\in[K]}\frac{1}{LK\alpha^{(t)}_i}$. If we also assume that all the loss functions are convex then the sequence $\{\theta^{(t)}\}_{t=1}^\infty$ converges to a Pareto optimal point $\theta^*$.
}


\begin{proof}
We note that this proof uses intermediate results from the proof of theorem \ref{trm:nonconvex}. Given theorem \ref{trm:nonconvex} it suffices to prove that the sequence $\{\theta^{(t)}\}_{t=1}^\infty$ converges, that would mean it converges to the partial limit $\theta^*$ that is Pareto stationary, and from convexity it would be Pareto optimal (as the optimizer of the convex combination of losses). For a convex and differential loss function, we have 
\begin{equation}\label{eq:convex}
\ell(\theta')\geq \ell(\theta)+\nabla \ell(\theta)\T (\theta'-\theta)    
\end{equation}
We can bound
\begin{align}
    ||\theta^{(t+1)}-\theta^*||^2 &=||\theta^{(t)}-\mu^{(t)}\Dt^{(t)}-\theta^*||^2\\
    &=||\theta^{(t)}-\theta^*||^2+(\mu^{(t)})^2||\Dt^{(t)}||^2-2\mu^{(t)}(\Dt^{(t)})\T (\theta^{(t)}-\theta^*)\label{eq:a1}\\
    &=||\theta^{(t)}-\theta^*||^2+(\mu^{(t)})^2 K-2\mu^{(t)}\sum_i\alpha^{(t)}_i(g^{(t)}_i)\T(\theta^{(t)}-\theta^*)\label{eq:a2}\\
    &\leq ||\theta^{(t)}-\theta^*||^2+(\mu^{(t)})^2 K+2\mu^{(t)}\sum_i\alpha^{(t)}_i(\ell_i(\theta^*)-\ell_i(\theta^{(t)}))\label{eq:a3}\\
    &\leq ||\theta^{(t)}-\theta^*||^2+(\mu^{(t)})^2 K+2\mu^{(t)}\sum_i\alpha^{(t)}_i(\ell_i(\theta^{(t+1)})-\ell_i(\theta^{(t)}))\label{eq:a4}\\
    &\leq ||\theta^{(t)}-\theta^*||^2+(\mu^{(t)})^2 K-2\mu^{(t)}\sum_i\alpha^{(t)}_i\frac{\mu^{(t)}}{2\alpha^{(t)}_i}\label{eq:a5}\\
    &= ||\theta^{(t)}-\theta^*||^2\label{eq:a6}
\end{align}
In Eq.~\ref{eq:a2} we use the definition of $\Dt^{(t)}$ and the fact that its norm equals $\sqrt{K}$. In Eq. \ref{eq:a3} we use convexity and Eq. \ref{eq:convex}. Eq. \ref{eq:a4} uses the fact that we show the losses are monotonically decreasing and converging to $\ell_i(\theta^*)$. In Eq. \ref{eq:a5} we use Eq. \ref{eq:b1}.

We have that the sequence $||\theta^{(t)}-\theta^*||$ is monotonically decreasing and bounded below by zero. Also, it has a subsequence that converges to zero, and so it must hold that the sequence $||\theta^{(t)}-\theta^*||$ also converge to zero, or equivalently $\theta^{(t)}\to \theta^*$.


\end{proof}

\textbf{Proposition (3.1).} \textit{
Denote the objective for the optimization problem in Eq.~\ref{eq:ccp_opt} by $\phi(\alpha)=\sum_i\beta_i(\alpha)+\varphi(\alpha)$. Then, $\phi\left(\alpha^{(\tau+1)}\right)\leq \phi\left(\alpha^{(\tau)}\right)$ for all $\tau\geq 1$.
}




\begin{proof}
In our concave-convex procedure, we use the following linearization at the $\tau$-th iteration:
$$
\tilde \varphi_{\tau}(\alpha)=\varphi(\alpha^{(\tau)} ) + \nabla\varphi(\alpha^{(\tau)} )\T(\alpha-\alpha^{(\tau)} ).
$$
Then,
\begin{align} \label{eq:1}
\tilde \varphi_{\tau}(\alpha^{(\tau)})=\varphi(\alpha^{(\tau)}).
\end{align}
Moreover, since $\varphi$ is concave and differentiable, we have that 
\begin{align} \label{eq:2}
\varphi(\alpha^{(\tau+1)} ) \le\varphi(\alpha^{(\tau)} ) + \nabla\varphi(\alpha^{(\tau)} )\T(\alpha^{(\tau+1)}-\alpha^{(\tau)} )=\tilde \varphi_{\tau}(\alpha^{(\tau+1)}).
\end{align}
Furthermore, since we minimize the convex objective $\sum_i \beta_i(\alpha)+\tilde \varphi(\alpha)$ at each iteration of our concave-convex procedure (in the convex feasible set),
\begin{align} \label{eq:3}
\sum_i \beta_i(\alpha^{(\tau)})+\tilde \varphi_{\tau}(\alpha^{(\tau)}) \ge \sum_i \beta_i(\alpha^{(\tau+1)})+\tilde \varphi_{\tau}(\alpha^{(\tau+1)}).
\end{align}
Using Eq. \ref{eq:1}--Eq. \ref{eq:3}, we have that
\begin{align*}
\phi(\alpha^{(\tau)}) = \sum_i \beta_i(\alpha^{(\tau)})+ \varphi(\alpha^{(\tau} ) = \sum_i \beta_i(\alpha^{(\tau)})+\tilde \varphi_{\tau}(\alpha^{(\tau)}) &\ge \sum_i \beta_i(\alpha^{(\tau+1)})+\tilde \varphi_{t}(\alpha^{(\tau+1)})
\\ & \ge \sum_i \beta_i(\alpha^{(\tau+1)}) + \varphi(\alpha^{(\tau+1)} )=\phi(\alpha^{(\tau+1)}).
\end{align*}
This proves the statement.  
  
\end{proof}

\section{Experimental Details}\label{sec:Appenix_exp_details}

We provide here full experimental details for all experiments described in the main text.

\textbf{Implementation Details.} We apply all gradient manipulation methods to the gradients of the shared weights, with the exception of IMTL-G, which was applied to the feature-level gradients, as was originally proposed by the authors. We also tried applying IMTL-G to the shared-parameters gradient for a fair comparison, but it did not perform as well. We set the CAGrad's $c$ hyperparameter to 0.4, which was reported to yield the best performance for NYUv2 and Cityscapes~\cite{liu2021conflict}. For DWA~\cite{Liu2019EndToEndML} we set the temperature hyperparameter to $2$ which was found empirically to
be optimum across all architectures. For RLW~\cite{lin2021closer} we sample the weights from a normal distribution.

\textbf{QM9.} We adapt the QM9 example in PyTorch Geometric~\cite{Fey/Lenssen/2019}, and train the popular GNN model from~\citet{gilmer2017neural}. We use the publicly available\footnote{\url{https://github.com/pyg-team/pytorch_geometric/blob/master/examples/qm9_nn_conv.py}} implementation, the implementation is provided by \citet{Fey/Lenssen/2019}. We use 110K molecules for training, 10K for validation, and 10K as a test set. Each task's targets are normalized to have zero mean and unit standard deviation. We train each method for $300$ epochs with batch-size of $120$ and search for learning-rate (lr) in $\{1e-3, 5e-4, 1e-4\}$. We use a ReduceOnPlateau scheduler to decrease the lr when the validation $\Delta_m$ metric stops improving. Additionally, we use the validation $\Delta_m$ for early stopping.

\textbf{Scene Understanding.} We follow the training and evaluation procedure used in previous work on MTL~\cite{Liu2019EndToEndML,yu2020gradient,liu2021conflict}. \revision{However, unlike~\citep{Liu2019EndToEndML}, we add data augmentations (DA) during training for all the compared methods, similar to~\cite{liu2021conflict,liu2020towards}.} We train each method for $200$ epochs with an initial learning-rate of $1e-4$. The learning-rate is reduced to $5e-5$ after $100$ epochs. 
For MTL methods, we train a Multi-Task Attention Network (MTAN) \citep{Liu2019EndToEndML} built upon SegNet~\cite{badrinarayanan2017segnet}. Similar to previous works~\cite{liu2021conflict}, the STL baseline refers to training task-specific SegNet models. We use a batch size of 2 and 8 for NYUv2 and CityScapes respectively. To align with previous work on MTL~\citet{Liu2019EndToEndML,yu2020gradient,liu2021conflict} we report the test performance averaged over the last $10$ epochs.

\textbf{MT10.} Following previous works~\cite{yu2020gradient, liu2021conflict,sodhani2021multi}, we use multitask Soft Actor-Critic (SAC)~\cite{haarnoja2018soft} as the base RL algorithm for PCGrad, CAGrad, and \ourmethod{}. We follow the same experiment setup
from and evaluation protocol as in~\citet{sodhani2021multi,liu2021conflict}. Each method is trained over 2 million steps with a batch size of 1280. 
The agent is evaluated once every $10$K environment
steps to obtain the average success over tasks. The reported success rate for the agent is the best average performance over all evaluation steps. We repeat this procedure over $10$ random seeds, and the performance of each method is obtained by averaging the mean success over all random seeds. For all \ourmethod{} experiments, we use a single CCP step in order to speed up computation.

\begin{figure*}[t]
\centering
    \begin{subfigure}[Average loss]{
    \includegraphics[width=0.25\linewidth]{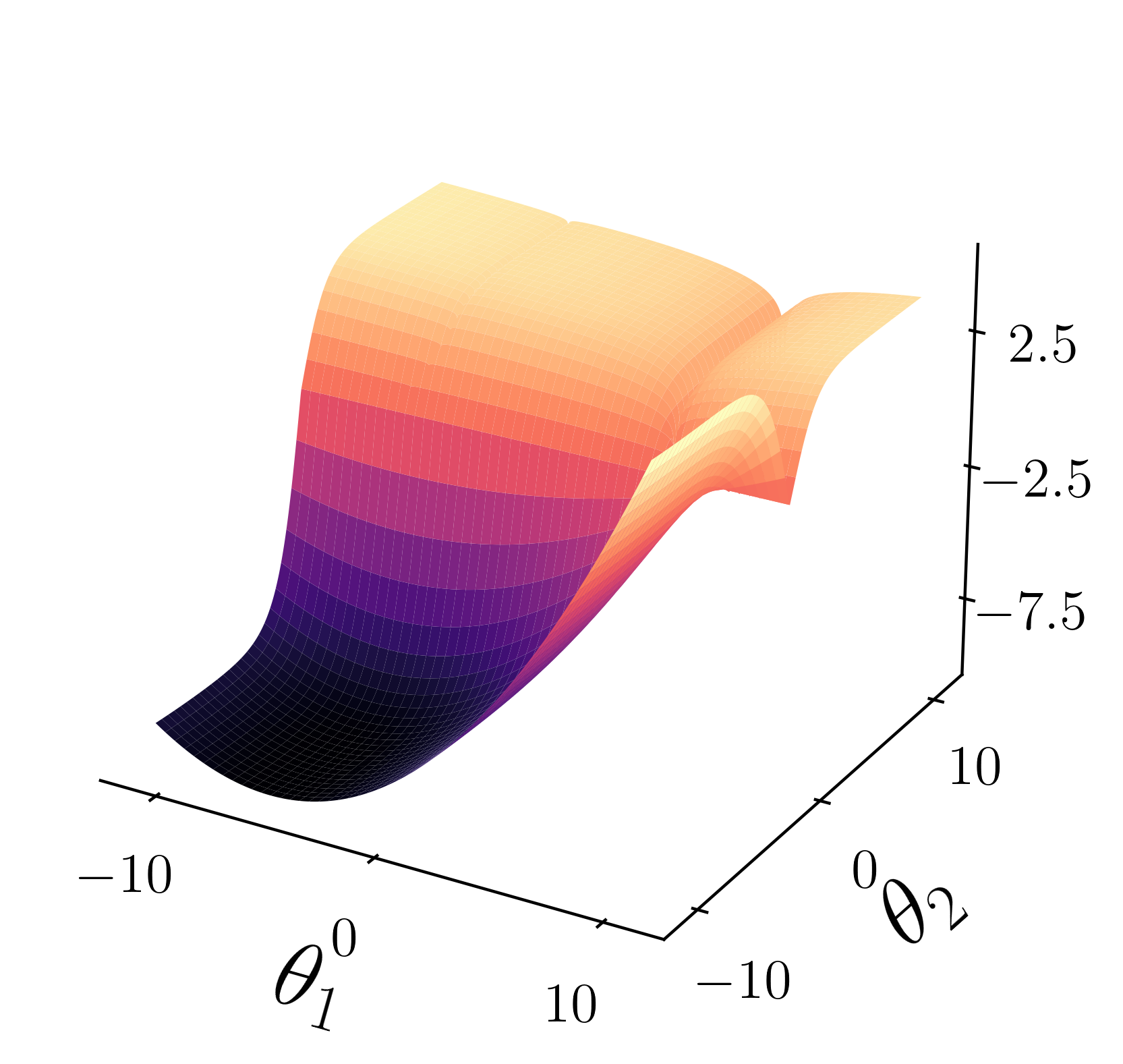}
    }
    \end{subfigure}
    \begin{subfigure}[$\ell_1$]{
    \includegraphics[width=0.25\linewidth]{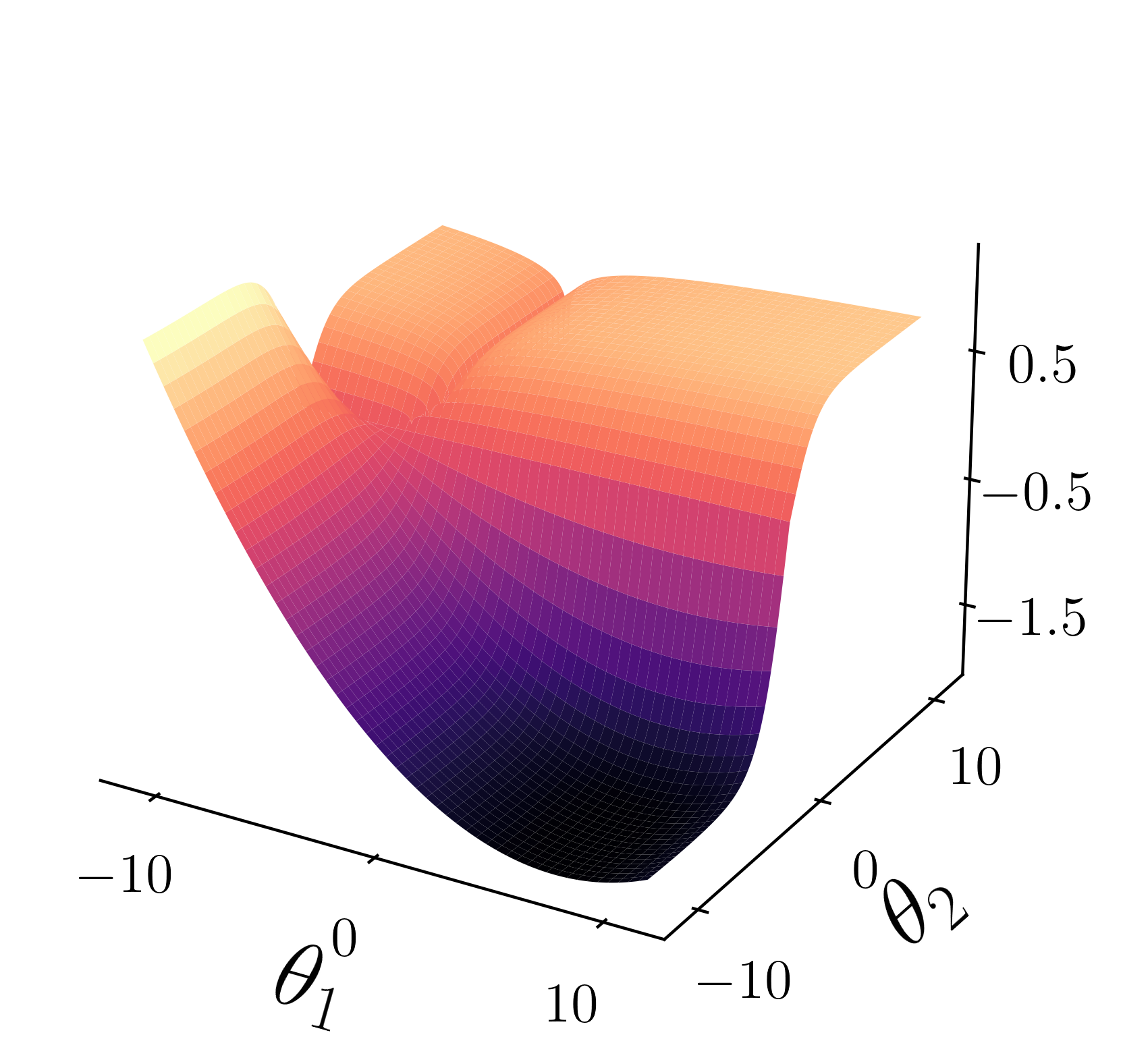}
    }
     \end{subfigure}
     \begin{subfigure}[$\ell_2$]{
    \includegraphics[width=0.25\linewidth]{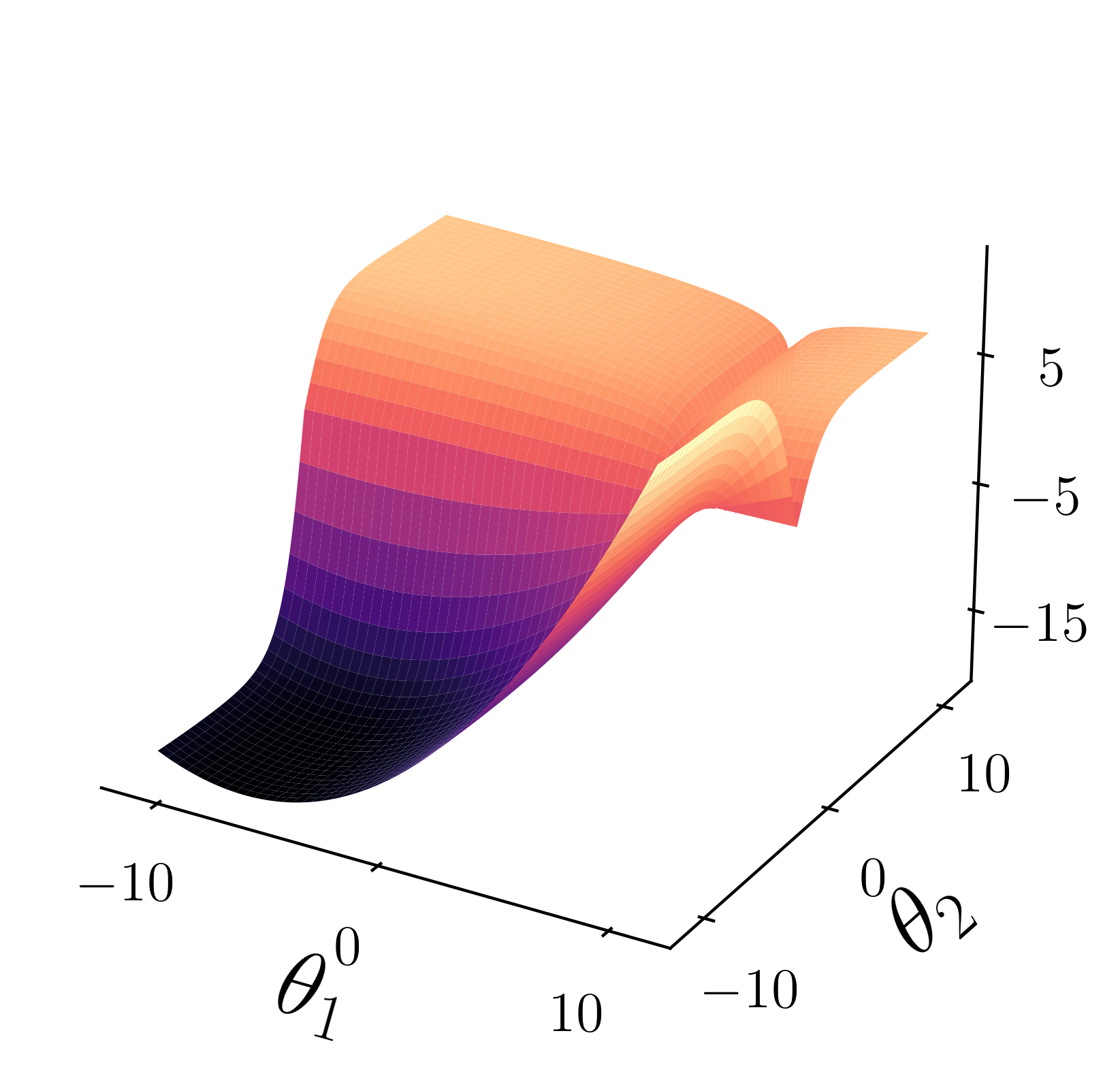}
    }
     \end{subfigure}
    \caption{\textit{Illustrative example}. Visualization of the loss surfaces in our illustrative example of Figure \ref{fig:toy}}
    \label{fig:illustrative_loss_landscape}
\end{figure*}

\textbf{Illustrative Example.}\label{sec:appendix_illustrative} 
We provide here the details for the illustrative example of Figure~\ref{fig:toy}. We use a slightly modified version of the illustrative example in~\cite{liu2021conflict}. We first present the learning problem from ~\cite{liu2021conflict}: Let $\theta=(\theta_1, \theta_2)\in \mathbb{R}^2$, and consider the following objectives:
\begin{align*}
    &\tilde{\ell}_1(\theta)=c_1(\theta)f_1(\theta) + c_2(\theta)g_1(\theta) \quad \text{and} \quad
    \tilde{\ell}_2(\theta)=c_1(\theta)f_2(\theta) + c_2(\theta)g_2(\theta), \text{where} \\
    &f_1(\theta)=\log (\max(|0.5(-\theta_1-7)-\tanh(-\theta_2)|,5e-6))+6,\\
    &f_2(\theta)=\log (\max(|0.5(-\theta_1+3)-\tanh(-\theta_2)+2|,5e-6))+6,\\
    &g_1(\theta)=((-\theta_1+7)^2+0.1\cdot (-\theta_2-8)^2)/10-20,\\
    &g_2(\theta)=((-\theta_1-7)^2+0.1\cdot (-\theta_2-8)^2)/10-20,\\
    &c_1(\theta)=\max(\tanh(0.5\theta_2), 0) \quad\text{and}\quad  c_2(\theta)=\max(\tanh(-0.5\theta_2), 0)\\
\end{align*}
We now set $\ell_1=0.1\cdot\tilde{\ell}_1$ and $\ell_2=\tilde{\ell}_2$ as our objectives, see Figure~\ref{fig:illustrative_loss_landscape}. We use five different initialization points $\{(-8.5, 7.5), (0.0, 0.0), (9.0, 9.0), (-7.5, -0.5), (9, -1.0)\}$. We use the Adam optimizer and train each method for 35K iteration with learning rate of $1e-3$.

\section{Computing Task Gradient at the Features-Level}\label{sec:Appendix_speedup_feature_level}

\revision{One common approach for speeding and scaling up MTL methods is using feature-level gradients (from the representation layer) as a surrogate for the task-level gradients computed over the entire shared backbone~\cite{sener2018multi,liu2020towards,javaloy2021rotograd}. In this section we evaluate \ourmethod{} while using the feature-level gradients for computing the Nash bargaining solution. On the QM9 dataset, we found this approach to accelerate training by $\sim\times 6$. However, this acceleration method greatly hurts the performance of \ourmethod{}, yielding a test $\Delta_m$ of $179.2$ (compared to $62.0$ when using full gradients). 
This result is not surprising, since we are mainly interested in the inner products of gradients. Consider $g_i\T g_j=(\nabla_{\theta}z\nabla_z\ell_i)\T \nabla_{\theta}z\nabla_z\ell_j$, where $z$ is the feature representation and $\theta$ the shared parameters vector. We see that for $\nabla_z\ell_i\T \nabla_z\ell_j$ to accurately approximate $g_i\T g_j$ we need $\nabla_{\theta}z\T \nabla_{\theta}z\approx I$ which is a strong and restricting requirement.}

\begin{figure*}[t]
\centering
    \includegraphics[width=0.47\linewidth]{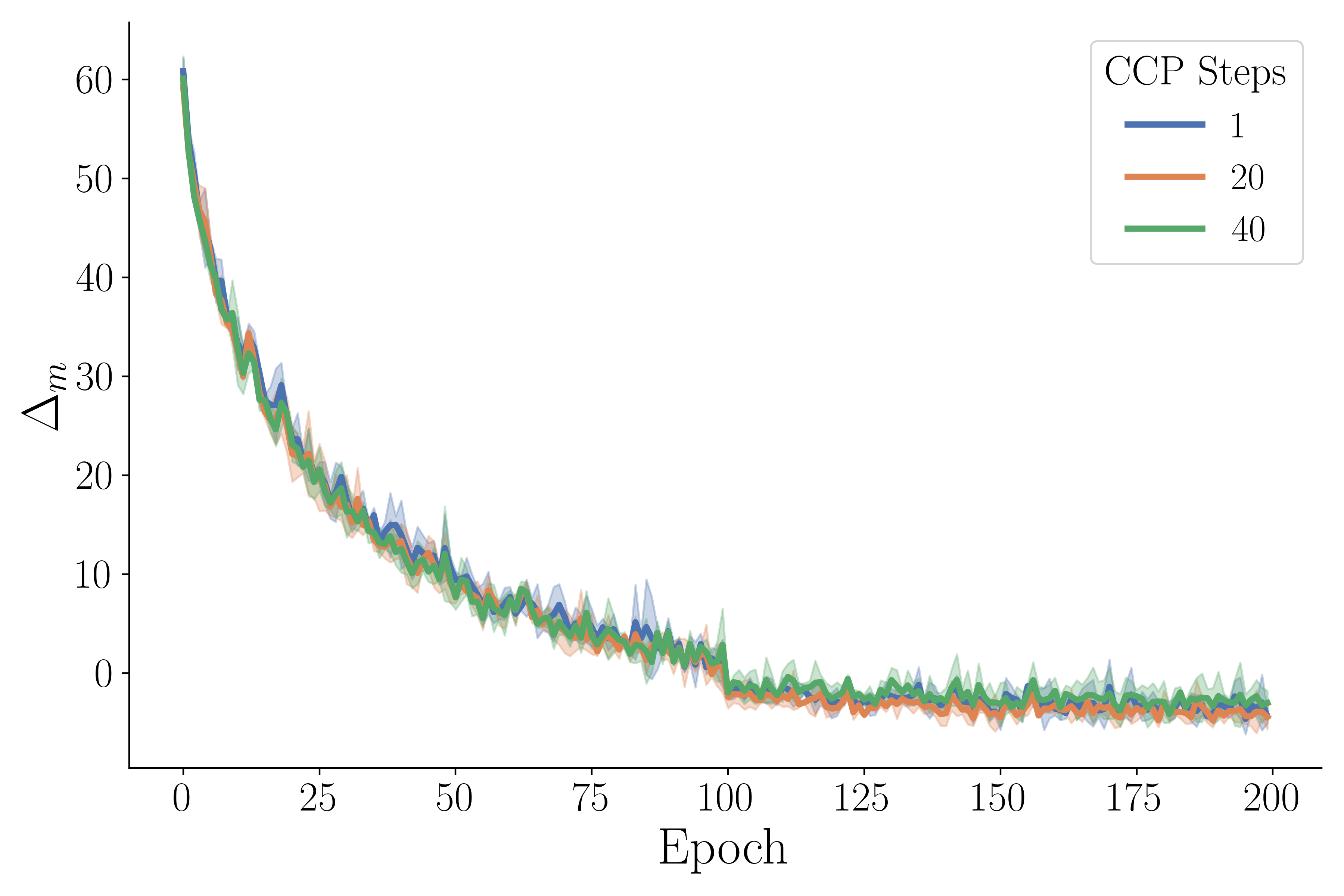}
    \caption{\textit{NYUv2.} The mean and standard divination of test $\Delta_m $ throughout the training process, for Nash-MTL with 1, 20, and 40 CCP steps.}
    \label{fig:nyu_ccp}
\end{figure*}

\section{Additional Experiments}\label{sec:Appenix_experiments}

\subsection{Full Results for Multi-task Regression}
\label{sec:Appendix_qm9}

We provide here the full results for the QM9 experiment of Section~\ref{sec:qm9}. The results for all methods over all $11$ tasks are presented in Table~\ref{tab:qm9_full}. \ourmethod{} achieves the best $\Delta_m$ and MR performance. Despite being a simple approach, \textit{SI} performs well compared to more sophisticated baselines. It achieves the third/second best $\Delta_m$ and MR respectively. The other scale-invariant method, \textit{IMTL-G}, also performs well in this learning setup.

\begin{table*}[!h]
\footnotesize
\centering
\caption{\textit{QM9}. Test performance averaged over 3 random seeds.}
    \vskip 0.11in
\begin{tabular*}{1.\textwidth}{@{\extracolsep{\fill}}clccccccccccccc@{}}
\toprule
 &  & $\mu$ & $\alpha$ & $\epsilon_{\text{HOMO}}$ & $\epsilon_{\text{LUMO}}$ & $\langle R^2 \rangle$ & ZPVE & $U_0$ & $U$ & $H$ & $G$ & $c_v$ &  \\ \cmidrule(lr){3-13}
 &  & \multicolumn{11}{c}{MAE $\downarrow$} & \textbf{MR} $\downarrow$ & $\mathbf{\Delta_m\%}$ $\downarrow$ \\ 
 \midrule
\multicolumn{2}{c}{STL} & $0.067$ & $0.181$ & $60.57$ & $53.91$ & $0.502$ & $4.53$ & $58.8$ & $64.2$ & $63.8$ & $66.2$ & $0.072$ &  \\ \midrule
\multicolumn{2}{c}{LS} & $0.106$ & $0.325$ & $\mathbf{73.57}$ & $89.67$ & $5.19$ & $14.06$ & $143.4$ & 144.2 & 144.6 & 140.3 & 0.128 & 6.8 & 177.6 \\
\multicolumn{2}{c}{SI} & 0.309 & 0.345 & 149.8 & 135.7 & $\mathbf{1.00}$ & $\mathbf{4.50}$ & $\mathbf{55.3}$ & $\mathbf{55.75}$ & $\mathbf{55.82}$ & $\mathbf{55.27}$ & 0.112 & 4.0 & 77.8 \\
\multicolumn{2}{c}{RLW} & 0.113 & 0.340 & 76.95 & 92.76 & 5.86 & 15.46 &  156.3 & 157.1 & 157.6 & 153.0 & 0.137 & 8.2 & 203.8 \\
\multicolumn{2}{c}{DWA} & 0.107 & 0.325 & 74.06 & 90.61 & 5.09 & 13.99 & 142.3 & 143.0 & 143.4 & 139.3 & 0.125 & 6.4 & 175.3 \\
\multicolumn{2}{c}{UW} & 0.386 & 0.425 & 166.2 & 155.8 & 1.06 & 4.99 & 66.4 & 66.78 & 66.80 & 66.24 & 0.122 & 5.3 & 108.0 \\
\multicolumn{2}{c}{MGDA} & 0.217 & 0.368 & 126.8 &  104.6 & 3.22 & 5.69 & 88.37 & 89.4 & 89.32 & 88.01 & 0.120 & 5.9 & 120.5 \\
\multicolumn{2}{c}{PCGrad} & 0.106 & 0.293 & 75.85 & 88.33 & 3.94 & 9.15 & 116.36 & 116.8 & 117.2 & 114.5 & 0.110 & 5.0 & 125.7 \\
\multicolumn{2}{c}{CAGrad} & 0.118 & 0.321 & 83.51 & 94.81 & 3.21 & 6.93 & 113.99 & 114.3 & 114.5 & 112.3 & 0.116 & 5.7 & 112.8 \\ 
\multicolumn{2}{c}{IMTL-G} & 0.136 & 0.287 & 98.31 & 93.96 & 1.75 & 5.69 & 101.4 & 102.4 & 102.0 & 100.1 & 0.096 & 4.7 & 77.2 \\ 
\midrule
\multicolumn{2}{c}{\ourmethod{}} & $\mathbf{0.102}$ & $\mathbf{0.248}$ & 82.95 & $\mathbf{81.89}$ & 2.42 & 5.38 & 74.5 & 75.02 & 75.10 & 74.16 & $\mathbf{0.093}$ & $\mathbf{2.5}$ & $\mathbf{62.0}$ \\ \bottomrule
\end{tabular*}
\label{tab:qm9_full}
\end{table*}


\subsection{Effect of the Number of CCP steps}
\label{sec:Appendix_ccp_steps}

In this section, we investigate the effect of varying the number of CCP steps in our efficient approximation to $G\T G\alpha=1/\alpha$ (presented in Section~\ref{sec:ccp}).  We use the NYUv2 dataset and train Nash-MTL with CCP sequences of 1, 20, and 40 steps at each (parameters) optimization step.

We found that increasing the CCP sequence improves the approximation to the optimal $\alpha$. Using a single CCP iteration results with $G\T G\alpha\approx 1/\alpha$ in $91.5\%$ of the optimization steps, whereas increasing the number of iterations to 20 increases the proportion of optimal solutions to $93.5\%$.
However, we found the improved solution to have no significant improvement in MTL performance. Figure~\ref{fig:nyu_ccp} presents the test $\Delta_m$ throughout the training process. 


\subsection{Modifying the CCP Objective}

\revision{In this section we examine the effect of changing the objective of the CCP procedure described in \ref{sec:ccp} (Eq.~\ref{eq:ccp_opt}). Here we first solve the convex optimization problem of Eq.~\ref{eq:ccp_opt_take_1} to obtain $\alpha_0$. If $G\T G \alpha_0\approx 1/\alpha_0$ we stop. Else we use the CCP procedure with objective $\varphi(\alpha)$, starting at $\alpha_0$ (dropping the addition $\sum_i \beta_i$ term from Eq.~\ref{eq:ccp_opt}). While this objective is more natural, in practice we observe a performance degradation in terms of MTL performance. We obtain $\Delta_m=64.4$ for the QM9 dataset (vs. $62$ reported in the paper), $\Delta_m=-3.5$ (vs. $-4$) for NYUv2 and $\Delta_m=8.8$ (vs. $6.8$) for Cityscapes.}

\subsection{Visualizing Task Weights}

Our method, Nash-MTL, can essentially be viewed as a principled approach for producing dynamic task weights. Here we visualize these task weights throughout the training process using the NYUv2 dataset (Figure~\ref{fig:task_weights}) and the MT10 dataset (Figure~\ref{fig:mt10_weights}).


\begin{figure*}[t]
\centering
    \begin{subfigure}[NYUv2]{
    \includegraphics[width=0.47\linewidth]{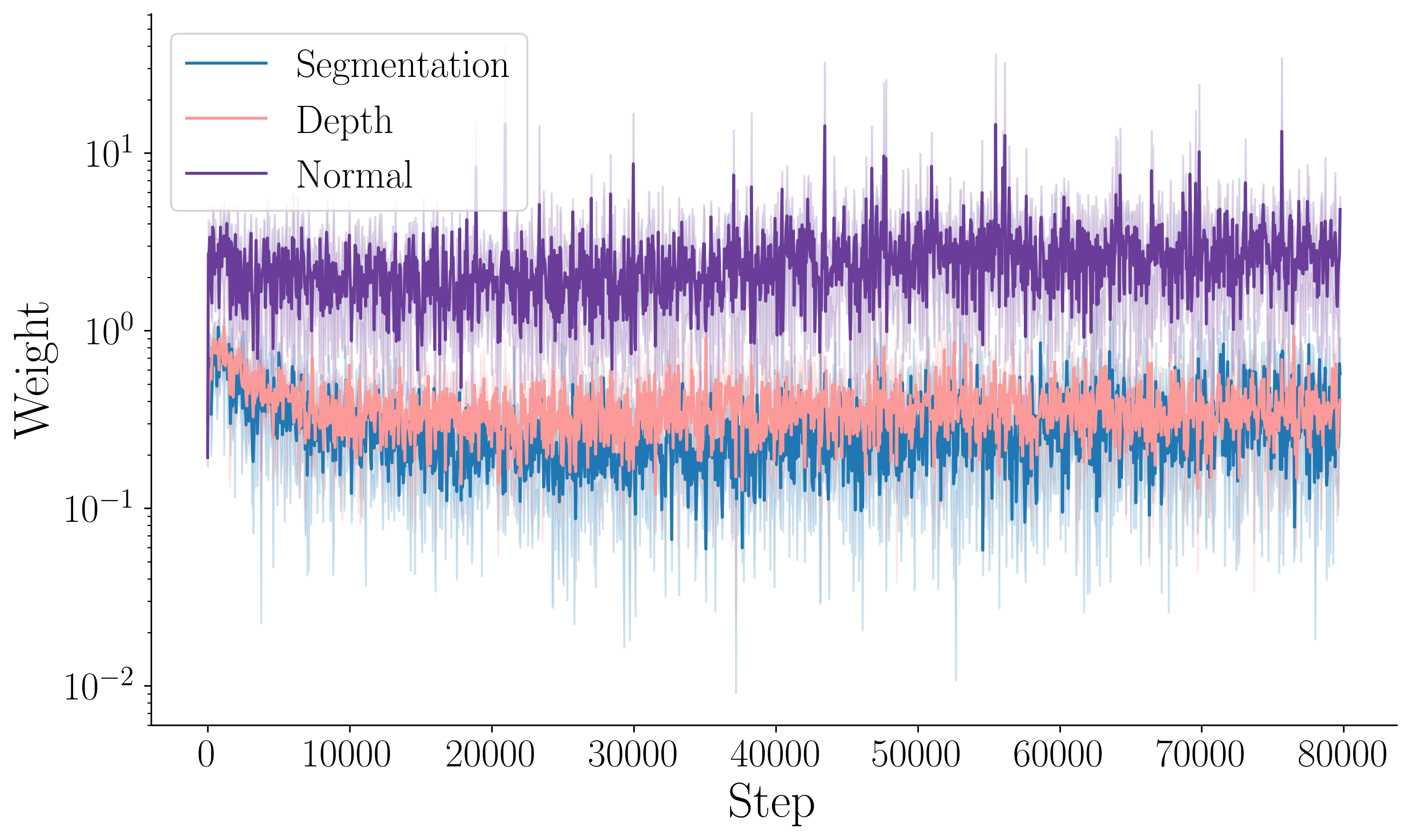}
    \label{fig:nyu_weights}
    }
    \end{subfigure}
    \begin{subfigure}[MT10]{
    \includegraphics[width=0.47\linewidth]{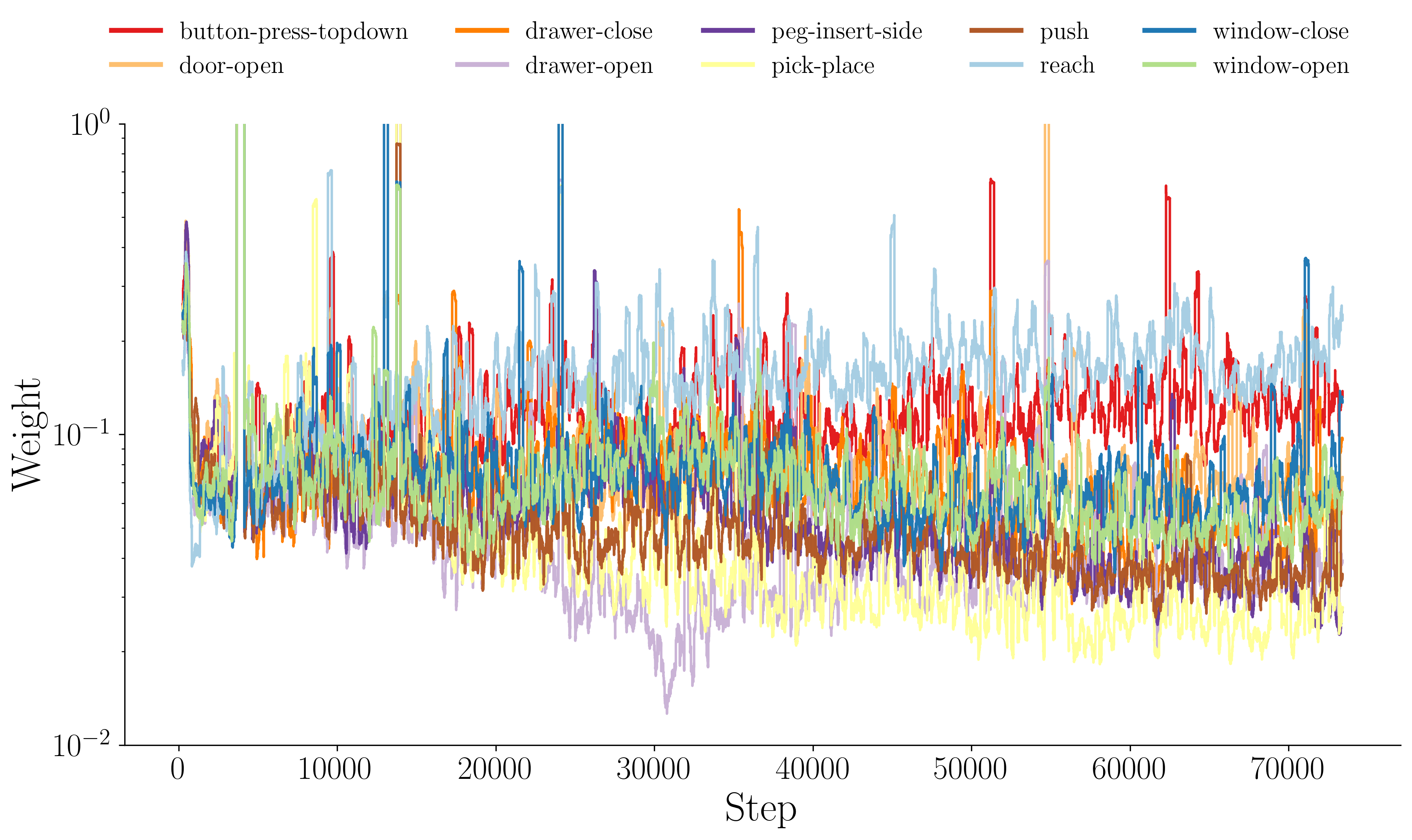}
    \label{fig:mt10_weights}
    }
     \end{subfigure}
    \caption{\textit{Task Weights}. Task weights obtained from \ourmethod{} throughout the optimization process, for \textbf{(a)} NYUv2, and; \textbf{(b)} MT10 with weight update frequency of 100. For better visualization, each point corresponds to a moving average with window size 200.}
    \label{fig:task_weights}
\end{figure*}

\subsection{Verifying the Task Independence Assumption}

Here we provide an empirical justification for our assumption in Section~\ref{sec:Method} which we state here once again: we assume that the task gradients are linearly independent for each point $\theta$ that is not Pareto stationary. To investigate whether this assumption holds in our experiments, we observe the smallest singular value of gradients Gram matrix $\sigma_K(G\T G)$. The results are presented in Figure~\ref{fig:sigma_min}. We see that for both datasets the $\sigma_k$ decreases as the learning progresses. For the NYUv2 experiment, the smallest singular value remains fairly large throughout the entire training process. On the QM9 dataset, $\sigma_K$ decreases more significantly, to around $\sim1e-8$. 

\begin{figure*}[ht]
\centering
    \begin{subfigure}[NYUv2]{
    \includegraphics[width=0.47\linewidth]{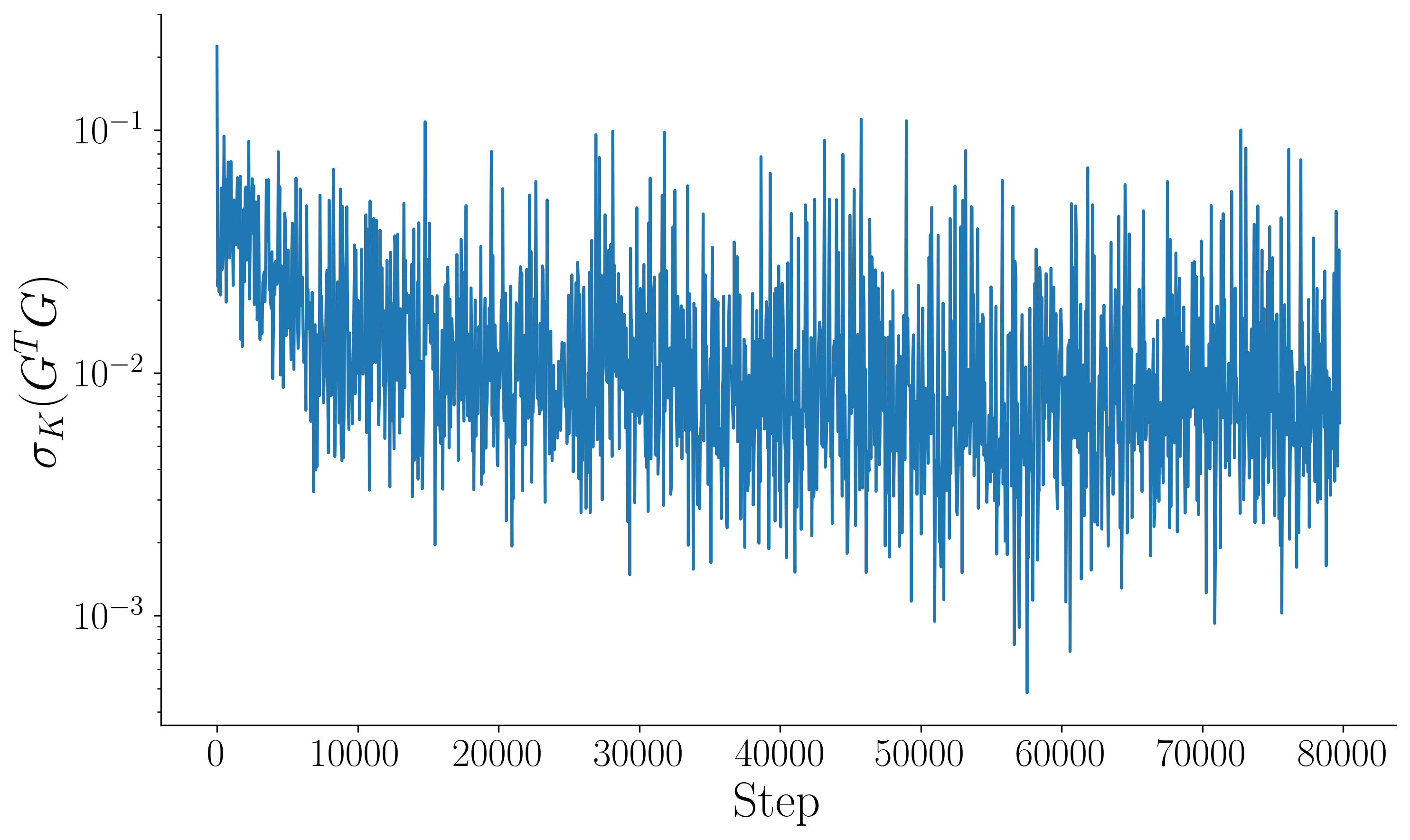}
    \label{fig:nyu_sigma_min}
    }
    \end{subfigure}
    \begin{subfigure}[QM9]{
    \includegraphics[width=0.47\linewidth]{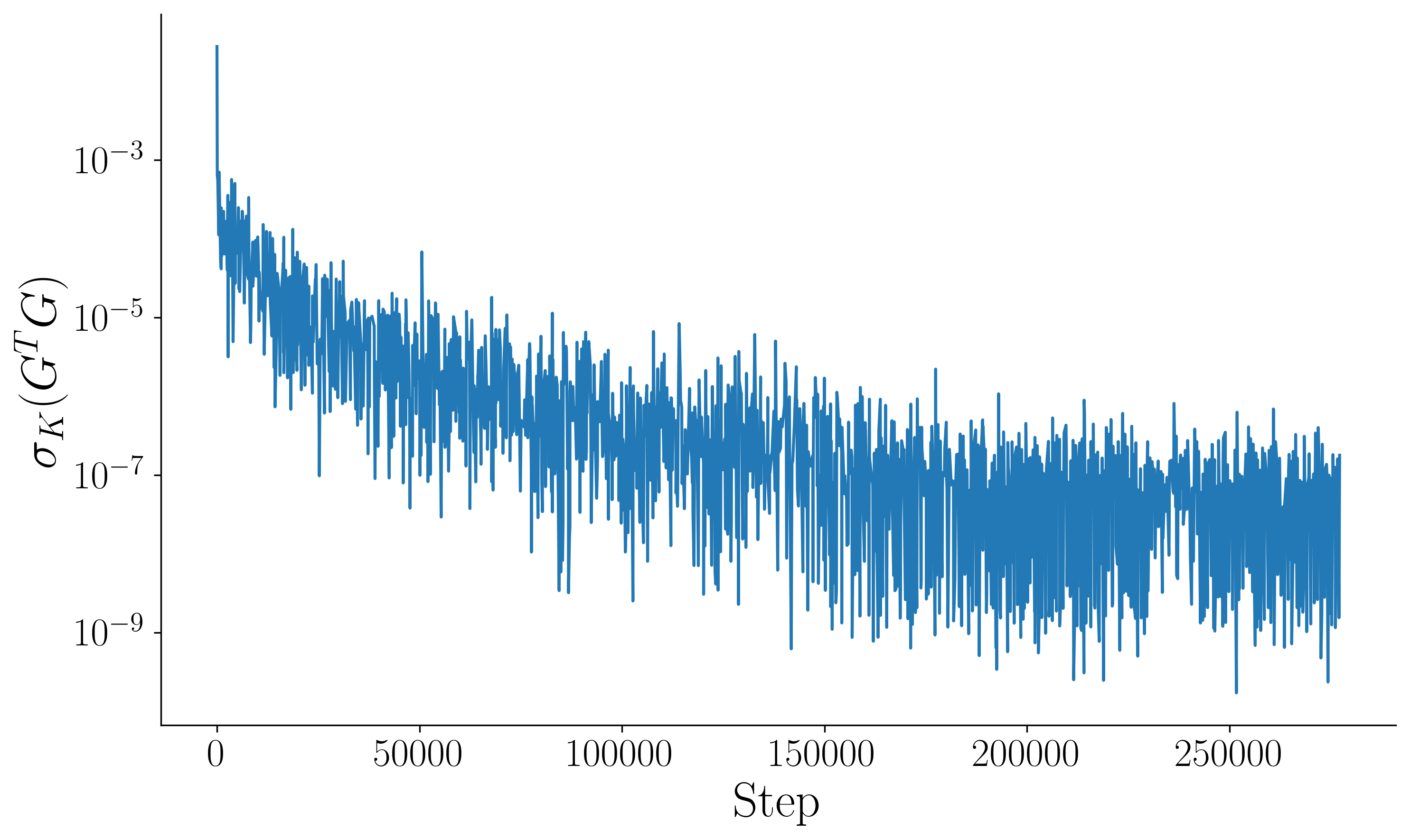}
    \label{fig:qm9_sigma_min}
    }
     \end{subfigure}
    \caption{\textit{Smallest singular value of $G\T G$ throughout the training process.}}
    \label{fig:sigma_min}
\end{figure*}

\begin{table}[!t]
\centering
\small

\caption{\textit{QM9}. Runtime per epoch in minutes.}
\vskip 0.11in

\begin{tabular}{@{}cccc@{}}
\toprule
                &                    & Runtime [Min.]   \\ \midrule
\multicolumn{2}{c}{LS}              & $0.54$ \\
\multicolumn{2}{c}{MGDA}           & $7.25$ \\
\multicolumn{2}{c}{PCGrad}           & $7.47$ \\
\multicolumn{2}{c}{CAGrad}           & $6.85$ \\
\midrule
\multicolumn{2}{c}{\ourmethod{}}     & $6.76$ \\
\multicolumn{2}{c}{\ourmethod{}-5}  & $1.81$ \\
\multicolumn{2}{c}{\ourmethod{}-50} & $0.69$ \\
\bottomrule
\end{tabular}
\label{tab:qm9_runtime}
\end{table}

\end{document}